\documentclass{article}

\usepackage[final]{neurips_2024}

\usepackage[utf8]{inputenc} 
\usepackage[T1]{fontenc}    
\usepackage{hyperref}       
\usepackage{url}            
\usepackage{booktabs}       
\usepackage{amsfonts}       
\usepackage{nicefrac}       
\usepackage{microtype}      
\usepackage{xcolor}         
\usepackage{natbib}
\usepackage{comment}
\usepackage[normalem]{ulem} 

\usepackage{mathtools}

\usepackage{amsmath}
\usepackage{amssymb}
\usepackage{amsthm}
\usepackage{multicol}
\usepackage{multirow}
\usepackage{cleveref}
\usepackage{bbm}

\theoremstyle{plain}
\newtheorem{theorem}{Theorem}[section]
\newtheorem{lemma}[theorem]{Lemma}

\newtheorem{corollary}[theorem]{Corollary}
\theoremstyle{definition}

\theoremstyle{remark}

\newcommand{\bW}{{\mathbf W}}
\newcommand{\bY}{{\mathbf Y}}
\newcommand{\bR}{{\mathbf R}}
\newcommand{\bA}{{\mathbf A}}
\newcommand{\bH}{{\mathbf H}}
\newcommand{\bI}{{\mathbf I}}
\newcommand{\bZ}{{\mathbf Z}}
\newcommand{\bC}{{\mathbf C}}

\newcommand{\bh}{{\mathbf h}}

\newcommand{\bzero}{{\mathbf 0}}
\newcommand{\bWT}{{\mathbf W}^{T}}

\newcommand{\bhi}{{\mathbf h}_i}
\newcommand{\by}{{\mathbf y}}
\newcommand{\byb}{{\bar{\mathbf y}}}
\newcommand{\yb}{{\bar{y}}}
\newcommand{\bYb}{{\bar{\mathbf Y}}}
\newcommand{\bx}{{\mathbf x}}
\newcommand{\bb}{{\mathbf b}}
\newcommand{\be}{{\mathbf e}}
\newcommand{\bc}{{\mathbf c}}
\newcommand{\bw}{{\mathbf w}}
\newcommand{\bv}{{\mathbf v}}
\newcommand{\bSigma}{{\mathbf \Sigma}}
\newcommand{\byi}{{\mathbf y}_i}
\newcommand{\bxi}{{\mathbf x}_i}
\newcommand{\lmin}{\lambda_{\text{min}}}
\newcommand{\lH}{\lambda_{\bH}}
\newcommand{\lW}{\lambda_{\bW}}

\newcommand{\Rd}{\mathbb{R}^d}

\title{The Prevalence of Neural Collapse \\in Neural Multivariate  Regression}

\author{%
  George Andriopoulos$^{1}$\thanks{Equal contribution.} \quad Zixuan Dong$^{2,4*}$ \quad Li Guo$^{3*}$ \quad Zifan Zhao$^{3*}$ \quad Keith Ross$^{1*}$\thanks{Corresponding author: \texttt{keithwross@nyu.edu}}\AND
  \text{\normalfont $^1$ New York University Abu Dhabi \quad
  $^2$ SFSC of AI and DL, NYU Shanghai}\\
$^3$ New York University Shanghai \quad
$^4$ New York University\\
}

\begin{document}

\maketitle

\begin{abstract}
Recently it has been observed that neural networks exhibit Neural Collapse (NC) during the final stage of training for the classification problem. We empirically show that multivariate regression, as employed in imitation learning and other applications, exhibits Neural Regression Collapse (NRC), a new form of neural collapse: (NRC1) The last-layer feature vectors collapse to the subspace spanned by the $n$ principal components of the feature vectors, where $n$ is the dimension of the targets (for univariate regression, $n=1$); (NRC2) The last-layer feature vectors also collapse to the subspace spanned by the last-layer weight vectors; (NRC3) The Gram matrix for the weight vectors converges to a specific functional form that depends on the covariance matrix of the targets. After empirically establishing the prevalence of (NRC1)-(NRC3) for a variety of datasets and network architectures, we provide an explanation of these phenomena by modeling the regression task in the context of the  Unconstrained Feature Model (UFM), in which the last layer feature vectors are treated as free variables when minimizing the loss function. We show that when the regularization parameters in the UFM model are strictly positive, then (NRC1)-(NRC3) also emerge as solutions in the UFM optimization problem. We also show that if the regularization parameters are equal to zero, then there is no collapse. To our knowledge, this is the first empirical and theoretical study of neural collapse in the context of regression. This extension is significant not only because it broadens the applicability of neural collapse to a new category of problems but also because it suggests that the phenomena of neural collapse could be a universal behavior in deep learning.
\end{abstract}

\section{Introduction}
\begin{figure}[htb]
    \centering
    \begin{minipage}{0.45\textwidth}
        \centering
        \includegraphics[width=\linewidth]{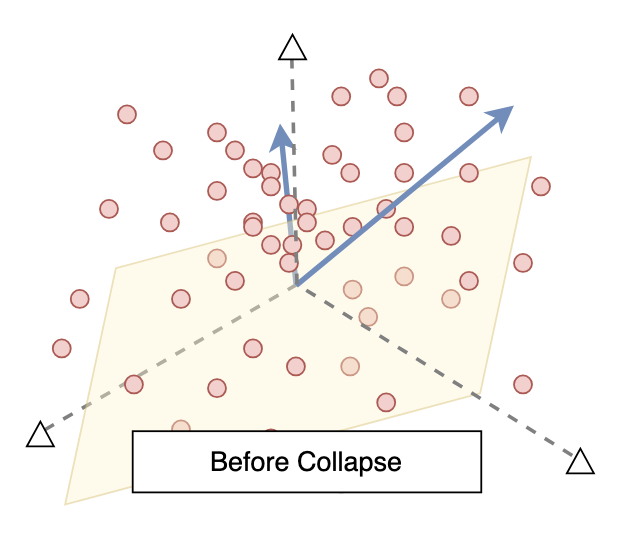}
    \end{minipage}\hfill 
    \begin{minipage}{0.44\textwidth}
        \centering
        \includegraphics[width=0.97\linewidth]{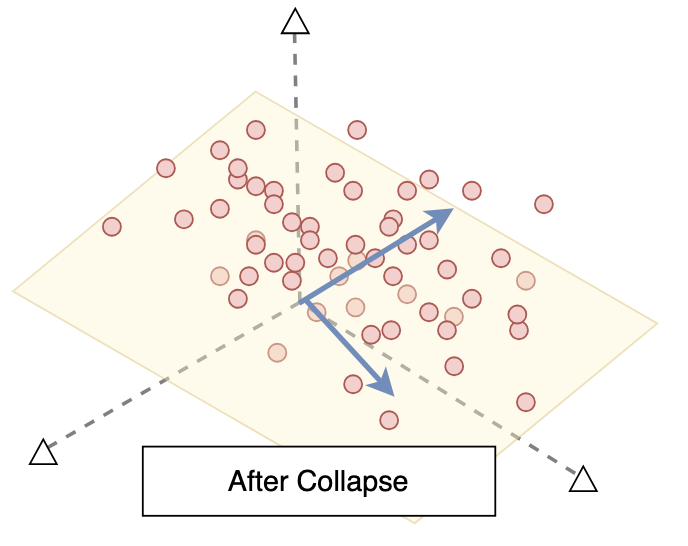}
    \end{minipage}
    \caption[Short Caption]{Visualization of the neural regression collapse. The red dots represent the sample features, the blue arrows represent the row vectors of the last layer weight matrix, and the yellow plane represents the plane spanned by the principal components of the sample features. Here the target dimension is $n=2$. The feature vectors and weight vectors collapse to the same subspace. The angle between the weight vectors takes specific forms governed by the covariance matrix of the targets. }
    \label{fig:full}
\end{figure}
Recently, an insightful phenomenon known as neural collapse (NC) \citep{papyan2020prevalence} has been empirically observed during the terminal phases of training in classification tasks with balanced data. 
NC has three principal components: (NC1) The features of samples within each class converge closely around their class mean. (NC2) The averages of the features within each class converge to form the vertices of a simplex equiangular tight frame. This geometric arrangement implies that class means are equidistant and symmetrically distributed. (NC3) The weight vectors of the classifiers in the final layer align with the class means of their respective features. 
These phenomena not only enhance our understanding of neural network behaviors but also suggest potential simplifications in the architecture and the training of neural networks.

The initial empirical observations of NC have led to the development of theoretical frameworks such as the layered-peeled model \citep{fang2021exploring} and the unconstrained feature model (UFM) \citep{mixon2020neural}. These models help explain why NC occurs in classification tasks theoretically. By allowing the optimization to freely adjust last-layer features along with classifier weights, these models provide important insights into the prevalence of neural collapse, showing that maximal class separability is a natural outcome for a variety of loss functions when the data is balanced \citep{han2021neural, poggio2020explicit, zhou2022optimization, zhou2022all}. 

Regression in deep learning is arguably equally important as classification, as it serves for numerous applications across diverse domains. In imitation learning for autonomous driving, regression is employed to predict continuous control actions (such as speed and steering angles) based on observed human driver behavior. Similarly, regression is used in robotics, where the regression model is trained to imitate expert demonstrations. 
In the financial sector, regression models are extensively used for predictive analytics, such as forecasting stock prices, estimating risk, and predicting market trends. 
Meteorology also heavily relies on regression models to forecast weather conditions. These models take high-dimensional inputs from various sensors and satellites to predict multiple continuous variables such as temperature, humidity, and wind speed. Moreover, many reinforcement learning algorithms include critical regression components, where regression is employed to predict value functions with the targets being Monte Carlo or bootstrapped returns. 

While NC has been extensively studied in classification, to our knowledge, its prevalence and implications in regression remain unexplored.  This paper investigates a new form of neural collapse within the context of neural multivariate regression. Analogous to the classification problem, we introduce Neural Regression Collapse (NRC):  (NRC1) During training, the last-layer feature vectors collapse to the subspace spanned by the $n$ principal components of the feature vectors, where $n$ is the dimension of the targets (for univariate regression, $n=1$); (NRC2) The last-layer feature vectors also collapse to the subspace spanned by the weight vectors; (NRC3) The Gram matrix for the weight vectors converges to a specific functional form that depends on the square-root of the covariance matrix of the targets. A visualization of NRC is shown in Figure \ref{fig:full}. 

Employing six different datasets -- including three robotic locomotion datasets, two versions of an autonomous driving dataset, and an age-prediction dataset -- and Multi-Layer Perceptron (MLP) and ResNet architectures, we establish the prevalence of NRC1-NRC3. 
This discovery suggests a universal geometric behavior extending beyond classification into regression models, simplifying our understanding of deep learning more generally.

To help explain these phenomena, we then apply the UFM model to neural multivariate regression with an L2 loss function. In this regression version of the problem, the optimization problem aims to minimize the regularized mean squared error over continuous-valued targets. 
We show that when the regularization parameters in the UFM model are strictly positive, then (NRC1)-(NRC3) also emerge as solutions in the UFM optimization problem,
thereby providing a mathematical explanation of our empirical observations. Among many observations, we discover empirically and theoretically that when the regression parameters are zero or very small, there is no collapse; and if we increase the parameters a small amount above zero, the (NRC1)-(NRC3) geometric structure emerges. 

To the best of our knowledge, this is the first empirical and theoretical study of neural collapse in the context of regression. 
By demonstrating the prevalence of neural collapse in regression tasks, we reveal that deep learning systems might inherently simplify their internal representations, irrespective of the specific nature of the task, whether it be classification or regression. 


\section{Related work}

Neural collapse (NC) was first identified by \cite{papyan2020prevalence} as a symmetric geometric structure observed in both the last layer features and classification vectors during the terminal phase of training of deep neural networks for classification tasks, particularly evident in balanced datasets. Since then, there has been a surge of research into both theoretical and empirical aspects of NC.

Several studies have investigated NC under different loss functions. For instance, \citep{han2021neural, poggio2020explicit, zhou2022optimization}  have observed and studied neural collapse under the Mean Squared Error (MSE) loss, while papers such as \citep{zhou2022all, guo2024cross} have demonstrated that label smoothing loss and focal loss also lead to neural collapse. In addition to the last layer, some papers \citep{he2023law, rangamani2023feature} have also examined the occurrence of the NC properties within intermediate layers.
Furthermore, beyond the balanced case, researchers have investigated the neural collapse phenomena in imbalanced scenarios. \citep{fang2021exploring} identified a phenomenon called minority collapse for training on imbalanced data, while \citep{hong2023neural,thrampoulidis2022imbalance, dang2023neural} offer more precise characterizations of the geometric structure under imbalanced conditions. 

To facilitate the theoretical exploration of the neural collapse phenomena, \citep{fang2021exploring, mixon2020neural} considered the unconstrained feature model (UFM). The UFM simplifies a deep neural network into an optimization problem by treating the last layer features as free variables to optimize over. 
This simplification is motivated by the rationale of the universal approximation theorem \citep{hornik1989multilayer}, asserting that sufficiently over-parameterized neural networks can be highly expressive and can accurately approximate arbitrary smooth functions. 
Leveraging the UFM, studies such as 
\citep{zhu2021geometric,zhou2022optimization, thrampoulidis2022imbalance, tirer2022extended, tirer2023perturbation, ergen2021revealing, wojtowytsch2020emergence}  have investigated models with different loss functions and regularization techniques. These studies have revealed that the global minima of the empirical risk function under UFMs align with the characterization of neural collapse observed by \citep{papyan2020prevalence}.
Beyond the UFM, some work \citep{tirer2022extended, sukenik2024deep} has extended the model to explore deep constrained feature models with multiple layers, aiming to investigate neural collapse properties beyond the last layer.

In addition to its theoretical implications, NC serves as a valuable tool for gaining deeper insights into DNN models and various regularization techniques \citep{guo2024cross, fisher2024pushing}. It provides crucial insights into the generalization and transfer learning capabilities of neural networks \citep{hui2022limitations, kothapalli2022neural, galanti2021role}, inspiring the design of enhanced model architectures for diverse applications. These include scenarios with imbalanced data \citep{yang2022inducing, kimfixed} and contexts involving online continuous learning \citep{seo2024learning}.

Despite extensive research on the neural collapse phenomena and its implications in classification, to the best of our knowledge,  there has been no investigation into similar issues regarding neural regression models.  Perhaps the paper closest to the current work is \citep{zhou2022optimization}, which applies the UFM model to the balanced classification problem with MSE loss. 
Although focused on classification, \citep{zhou2022optimization} 
derive some important results which apply to regression as well as to classification. 
Our UFM analysis leverages this related paper, particularly their Lemma B.1.


\section{Prevalence of neural regression collapse}

We consider the multivariate regression problem with $M$ training examples $\{(\bxi,\byi), i=1,\ldots,M \}$, where each input $\bxi$ belongs to $\mathbb{R}^D$ and each 
target vector $\byi$ belongs to $\mathbb{R}^n$. 
For the regression task, the deep neural network (DNN) takes as input an example $\bx \in \mathbb{R}^D$ and produces an output $\by = f(\bx) \in \mathbb{R}^n$. For most DNNs, including those used in this paper, this mapping takes the form
$f_{\theta,\bW,\bb}(\bx)=\bW \bh_{\mathbf{\theta}}(\bx) + \bb$,
where $\bh_{\mathbf{\theta}}(\cdot): \mathbb{R}^D\to \Rd$ is the non-linear feature extractor consisting of several nonlinear layers, $\bW$ is a $n \times d$ matrix representing the final linear layer in the model, and $\bb \in \mathbb{R}^n$ is the bias vector. 
For most neural regression tasks, $n << d$, that is the dimension of the target space is much smaller than the dimension of the feature space. For univariate regression, $n=1$. 
The parameters ${\mathbf{\theta}}$, $\bW$, and $\bb$ are all trainable.

We train the DNN using gradient descent to minimize the regularized L2 loss: 
\[
\min_{\mathbf{\theta}, \bW, \bb} \frac{1}{2M} \sum_{i=1}^{M} ||f_{\theta,\bW,\bb}(\bxi) - \byi||_2^2 + \frac{\lambda_{\theta}}{2} ||\mathbf{\theta}||_2^2 +  \frac{\lambda_{\bW}}{2}||\bW||_F^2,
\]
where $||\cdot||_2$ and $||\cdot||_F$ denote the $L_2$-norm and the Frobenius norm, respectively.
As commonly done in practice,  in our experiments we set all the regularization parameters to the same value, which we refer to as the weight-decay parameter $\lambda_{WD}$, that is, we set $\lambda_{\theta} = \lambda_{\bW} =\lambda_{WD}$. 

\subsection{Definition of neural regression collapse}

In order to define Neural Regression Collapse (NRC), let  $\bSigma$ denote the $n \times n$ covariance matrix corresponding to the targets $\{\byi, i=1,\ldots,M \}$: $\bSigma = M^{-1}(\bY - \bYb) (\bY - \bYb)^{T}$, where $\bY = [\by_1 \cdots \by_M]$, $\bYb = [\byb \cdots \byb]$, and $\byb = M^{-1} \sum_{i=1}^M \byi$.
Throughout this paper, we make the natural assumption that $\bY$ and $\bSigma$ have full rank. Thus $\bSigma$ is positive definite. 
Let $\lmin >0$ denote the minimum eigenvalue of $\bSigma$. 

Denote $\bH := [\bh_1 \cdots \bh_M$], where $\bhi$ is the feature vector associated with input $\bxi$, that is,  $\bhi := \bh_{\mathbf{\theta}}(\bx_i)$. Further denote the normalized feature vector $\widetilde{\bh}_i := \bhi\cdot||\bhi||^{-1}$. Of course, $\bW$, $\bH$, and $\bb$ are changing throughout the course of training. 
For any $p \times q$ matrix $\bC$ and any $p$-dimensional vector $\bv$, let $proj(\bv|\bC)$ denote the projection of $\bv$ onto the subspace spanned by the columns of $\bC$. 
Let $\bH_{PCA_n}$ be the $d \times n$ matrix with the columns consisting of the $n$
principal components of $\bH$.  


We say that {\em Neural Regression Collapse (NRC)} emerges during training if the following three phenomena occur: 
\begin{itemize}
    \item NRC1 = $\displaystyle \frac{1}{M} \sum_{i=1}^M \left|\left|\widetilde{\bh}_i - proj(\widetilde{\bh}_i |\bH_{PCA_n}) \right|\right|_2^2\to 0$.
    \item NRC2 = $\displaystyle \frac{1}{M}\sum_{i=1}^M \left|\left|\widetilde{\bh}_i - proj(\widetilde{\bh}_i |\bW^T) \right|\right|_2^2 \to 0$.
    \item There exists a constant $ \gamma \in (0, \lambda_{\min})$ such that:
\[
\text{NRC3} = \left|\left|\frac{\bW\bWT}{||\bW\bWT||_F} - \frac{ \bSigma^{1/2} - \gamma^{1/2} \bI_n }{||\bSigma^{1/2} - \gamma^{1/2} \bI_n||_F }\right|\right|_F^2 \to 0.
\]
\end{itemize}
NRC1 $\to 0$ indicates that there is {\em feature-vector collapse}, that is, the $d$-dimensional feature vectors $\bhi$, $i=1,\ldots,M$, collapse to a much  lower $n$-dimensional subspace spanned by their $n$ principal components. In many applications, $n=1$, in which case the feature vectors are collapsing to a line in the original $d$-dimensional space.
NRC2 $\to 0$ indicates that there is a form of {\em self duality}, that is, the feature vectors also collapse to the $n$-dimensional space spanned by the rows of $\bW$.  
NRC3 $\to 0$ indicates that the last-layer weights have a {\em specific structure} within the collapsed subspace. In particular, it gives detailed information about the norms of the row vectors in $\bW$ and the angles between those row vectors. NRC3 $\to 0$ indicates that angles between the rows in $\bW$ are influenced by $\bSigma^{1/2}$. If the targets are uncorrelated so that $\bSigma$ and $\bSigma^{1/2}$ are diagonal, then NRC3 $\to 0$ implies that the rows in $\bW$ will be orthogonal.  
NRC3 $\to 0$ also implies a specific structure for the feature vectors, as discussed in Section 4. 

\subsection{Experimental validation of neural regression collapse}
\label{sec:exp_case1}

In this section, we validate the emergence of NRC1-NRC3 during training across various datasets and deep neural network (DNN) architectures.

\textbf{Datasets}. The empirical experiments in this section are based on the following datasets:
\begin{itemize}
\item The {\bf Swimmer}, {\bf Reacher}, and {\bf Hopper datasets} are based on MoJoCo \citep{mujoco, brockman2016openai,towers_gymnasium_2023}, a physics engine that simulates diverse continuous multi-joint robot controls and has been a canonical benchmark for deep reinforcement learning research. In our experiments, we use publicly available expert datasets (see appendix \ref{appendix:mujoco}).
Each dataset comprises raw robotic states as inputs ($\bxi$'s) and robotic actions as targets ($\byi$'s). In order to put these expert datasets in an imitation learning context, we reduced the size of the dataset by keeping only a small portion of the episodes. 
\item The {\bf CARLA dataset} originates from the CARLA Simulator, an open-source project designed to support the development of autonomous driving systems. We utilize a dataset \cite{Codevilla2018} sourced from expert-driven offline simulations. During these simulations, images ($\bxi$'s) from cameras mounted on the virtual vehicle and corresponding expert driver actions as targets ($\byi$'s) are recorded as human drives in the simulated environment. We consider two dataset versions: a 2D version with speed and steering angle, and a 1D version with only the speed.
\item The {\bf UTKFace dataset} \citep{zhifei2017cvpr} is widely used in computer vision to study age estimation from facial images of humans. This dataset consists of about 25,000 facial images spanning a wide target range of ages, races, and genders.  
\end{itemize}

Table \ref{tab:data} summarizes the six datasets, with the dimensions of the target vectors $\by$ ranging from one to three. The table also includes the minimum eigenvalue of the associated covariance matrix $\bSigma$ and the Pearson correlation values between the $i$-th and $j$-th target components for $i \neq j$. When $n=1$, there is no correlation value; when $n=2$, there is one correlation value between the two target components; and when $n=3$, there are three correlation values among the three target components. From the table, we observe that the target components in CARLA 2D and Reacher are nearly uncorrelated, whereas those in Hopper and Swimmer exhibit stronger correlations.

\begin{table}[h!]

\label{table:datasets}
\caption{Overview of datasets employed in our neural regression collapse analysis.}
\begin{center}
\begin{small}
\begin{tabular}{ccccccc}
\toprule
\textbf{Dataset} & \textbf{Data Size} & \textbf{Input Type} & \textbf{Target Dimension} $n$ & \textbf{Target Correlation} & $\mathbf{\lambda_{min}}$\\
\midrule
Swimmer & 1,000 & raw state & 2 & -0.244 & 0.276\\

Reacher & 1,000 & raw state & 2 & -0.00933 & 0.0097\\

Hopper & 10,000 & raw state & 3  & [-0.215, -0.090, 0.059]  & 0.215\\
\midrule
Carla 1D & 600,000 & RGB image & 1 & NA & 208.63\\
\midrule
Carla 2D & 600,000 & RGB image & 2 & -0.0055 & 0.156\\
\midrule
UTKface & 25,000 & RGB image & 1 & NA & 1428\\

\bottomrule
\end{tabular}
\end{small}
\end{center}
\label{tab:data}
\end{table}

\textbf{Experiment Settings}. For the Swimmer, Reacher, and Hopper datasets, we employed a four-layer MLP (with the last layer being the linear layer) as the policy network for the prediction task. Each layer consisted of 256 nodes, aligning with the conventional model architecture in most reinforcement learning research \citep{tarasov2022corl}. 
For the CARLA and UTKFace datasets, we employed ResNet18 and ResNet34 \cite{he2016deep}, respectively. To focus on behaviors associated with neural collapse and minimize the influence of other factors, we applied standard preprocessing without data augmentation. \looseness=-1

All experimental results are averaged over at least 2 random seeds and variance is displayed by a shaded area. The choices of weight decay employed during training varied depending on the dataset. Also, the number of epochs required for training depends on both the dataset and the degree of weight decay. In particular, we used a large number of epochs when using very small weight decay values. appendix \ref{sec:a_exp} provides the full experimental setup.

\textbf{Empirical Results}. Figure \ref{NRC_Case1} presents the experimental results for the six datasets mentioned above. The results show that the training and testing errors decrease as training progresses, as expected. The converging coefficient of determination ($R^2$) also indicates that model performance becomes stable. Most importantly, the figure confirms the presence of NRC1-NRC3 across all six datasets. This indicates that neural collapse is not only prevalent in classification but also often occurs in multivariate regression. 
\begin{figure}[h]
    \centering
    \includegraphics[width=1\linewidth]{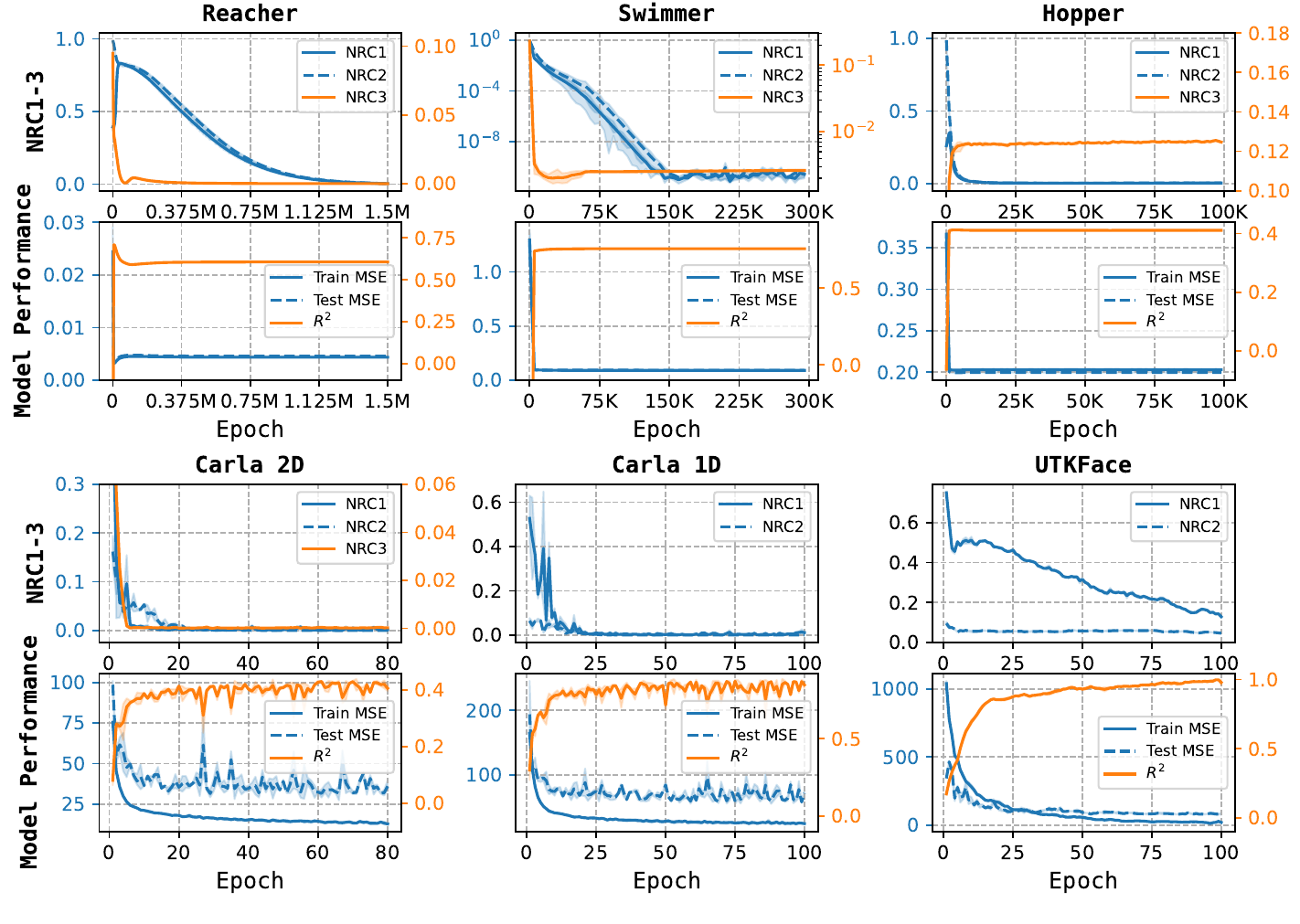}
    \caption{Prevalence of NRC1-NRC3 in the six datasets. Train/Test MSE and the coefficient of determination ($R^2$) are also shown.}
    \label{NRC_Case1}
\end{figure}
\begin{figure}[h]
    \centering
    \includegraphics[width=1\linewidth]{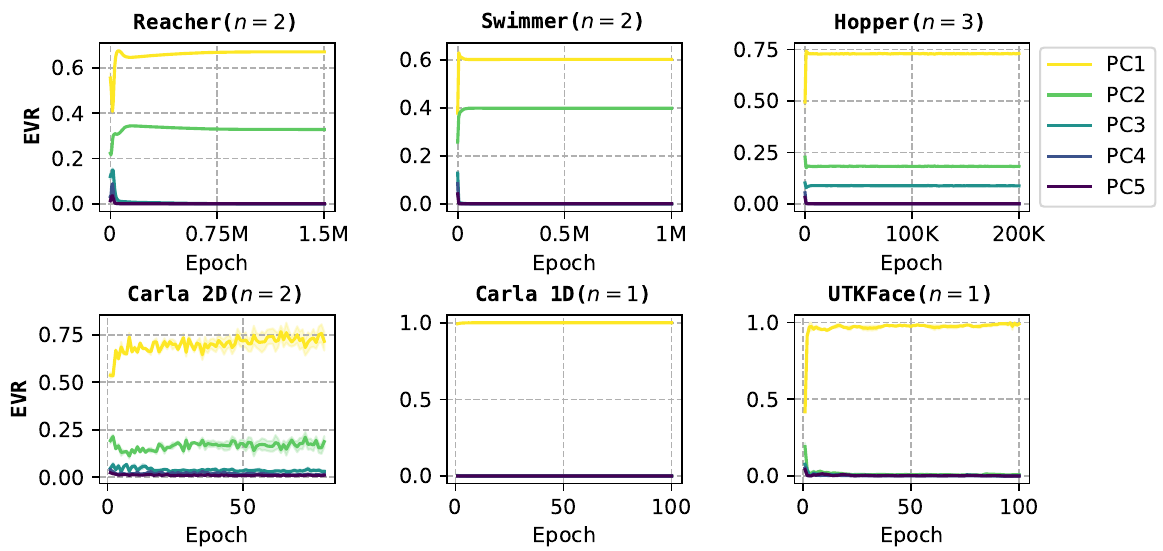}
    \caption{Explained Variance Ratio (EVR) for the first 5 principal components (PC).}
    \label{fig:evr}
\end{figure}

We also experimentally analyze the explained variance ratio (EVR) of principal components to further verify the collapse to the subspace spanned by the first $n$ components. In Figure \ref{fig:evr}, we investigate the EVR of the first 5 principal components of $\bH$ during the training process. For all datasets, there is significant variance for all of the first $n$ components after a short period of training; for other components, there is very low or even no variance. This also supports that a perfect collapse occurs in the subspace spanned by the first $n$ principal components.

Our definition of NRC3 involves finding a scaling factor $\gamma$ for which the property holds. Figure \ref{optimal_c} illustrates the values of NRC3 as a function of $\gamma$ for $\bW$ obtained after training. We observe that each dataset exhibits a unique minimum value of $\gamma$. More details about computing NRC3 can be found in appendix \ref{appendix: nrc3}.
\begin{figure}[htb]
    \centering
    \includegraphics[width=\linewidth]{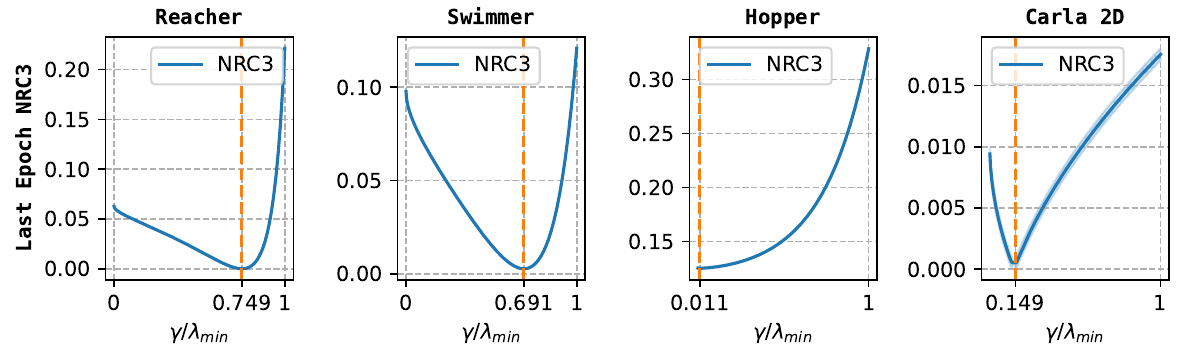}
    \caption{The optimal value of $\gamma$ for NRC3.}
    \label{optimal_c}
\end{figure}

Figure \ref{small_reg} investigates neural regression collapse for small values of the weight-decay parameter $\lambda_{WD}$. (appendix \ref{appendix: full_fig4} contains results on all 6 datasets.) We see that when $\lambda_{WD} = 0$, there is no neural regression collapse. However, if we increase $\lambda_{WD}$ by a small amount, collapse emerges for all three metrics. Thus we can conclude that the geometric structure NRC1-3 that emerges during training is due to regularization, albeit the regularization can be very small. In the next section, we will introduce a mathematical model that helps explain why there is no collapse when  $\lambda_{WD} = 0$ and why it quickly emerges as $\lambda_{WD}$ is increased above zero.

\begin{figure}[hb]
    \centering
    \includegraphics[width=\linewidth]{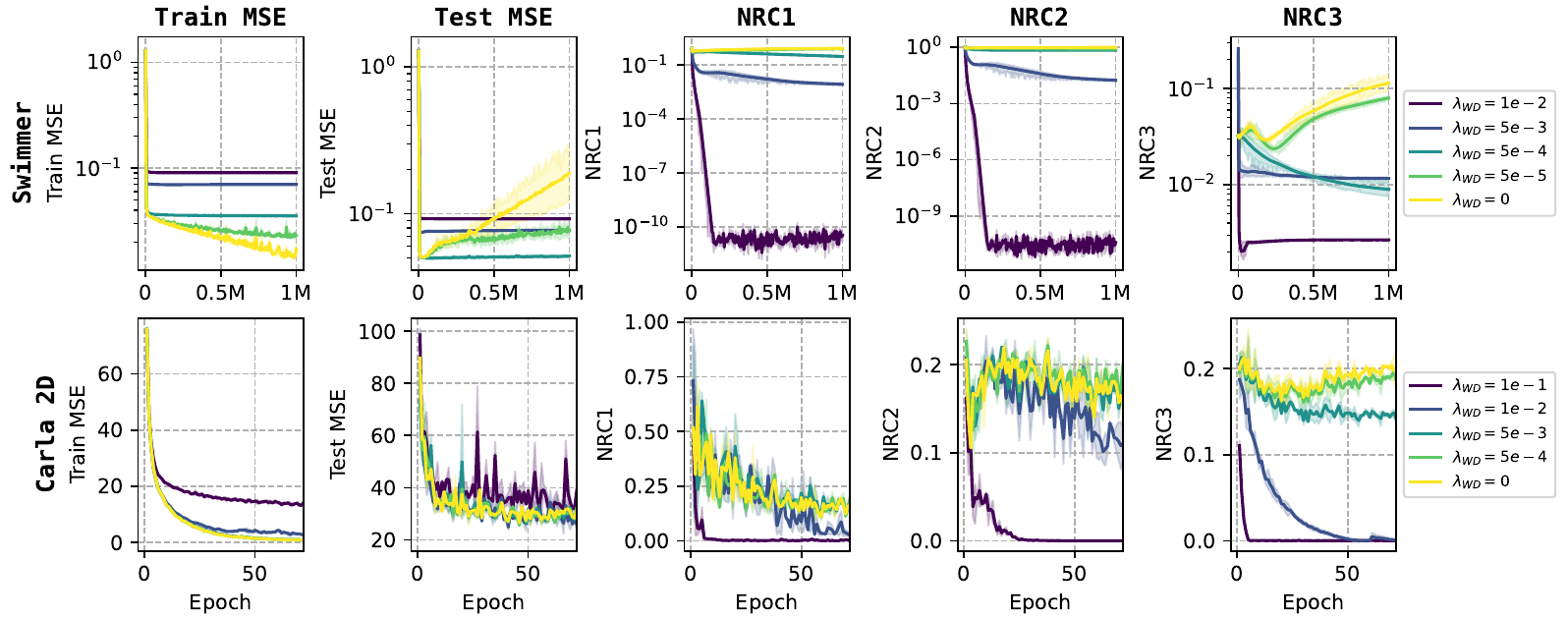}
    \caption{Phase change in neural collapse for small weight-decay values}
    \label{small_reg}
\end{figure}

\section{Unconstrained feature model} \label{sec: theoretical}
As discussed in the related work section, the UFM model has been extensively used to help explain the prevalence of neural collapse in the classification problem. In this section, we explore whether the UFM model can also help explain neural collapse in neural multivariate regression. 

Specifically, we consider minimizing $\mathcal{L}(\bH, \bW,\bb)$, where
\begin{align} \label{formofloss}
   \mathcal{L}(\bH, \bW,\bb) 
    & = \frac{1}{2M} ||\bW \bH + \bb \mathbf{1}_M^{T} - \bY||_F^2 + \frac{\lH}{2M} ||\bH||_{F}^2 + \frac{\lW}{2} ||\bW||_F^2,
\end{align}
where $\mathbf{1}_M^T:=[1\cdots 1]$ and $\lH$, $\lW$ are non-negative regularization constants. 

The optimization problem studied here bears some resemblance to the standard linear multivariate regression problem. If we view the features $\bh_i$, $i=1,\ldots,M$, as the inputs to linear regression, then $\hat{\by}_i := \bW \bh_i + \bb$ is the predicted output, and $||\byi -  \hat{\by}_i||_2^2$ is the squared error. In standard linear regression, the $\bhi$'s are fixed inputs. In the UFM model, however, not only are we optimizing over the weights $\bW$ and biases $\bb$ but also  over all the ``inputs'' $\bH$. 

For the case of classification, regularization is needed in the UFM model to prevent the norms of $\bH$ and/or $\bW$ from going to infinity in the optimal solutions. In contrast, in the UFM regression model, the norms in the optimal solutions will be finite even without regularization. However, as regularization is typically used in neural regression problems to prevent overfitting, it is useful to include regularization in the UFM regression model as well. 

\subsection{Regularized loss function}

Throughout this subsection, we assume that both $\lW$ and $\lH$ are strictly positive. We shall consider the $\lW = \lH = 0$ case subsequently. 
We also make a number of assumptions in order to not get distracted by less important sub-cases. 
Throughout we assume $n \leq d$, that is, the dimension of the targets is not greater than the dimension of the feature space. As stated in a previous subsection, for problems of practical interest, we have $n << d$.
Recall that $\bSigma$ is the covariance matrix of the target data. 
Since $\bSigma$ is a covariance matrix and is assumed to have full rank, it is also positive definite. It therefore has a positive definite square root, which we denote by $\bSigma^{1/2}$.  
Let $\lambda_{\max} := \lambda_1 \geq \lambda_2 \geq \cdots \geq \lambda_n := \lambda_{\min}> 0$ denote the $n$ eigenvalues of $\bSigma$. 
We further define the $n \times n$ matrix 
\begin{equation} \label{defofa}
{\bf A} := \bSigma^{1/2} -\sqrt{c}\bI_n,
\end{equation}
where $c := \lW \lH$. Also for any $p \times q$ matrix $\bC$ with columns $\bc_1, \bc_2,\ldots,\bc_q$, we denote $[\bC]_j$ to be 
the $p \times q$ matrix whose first $j$ columns are identical to those in $\bC$ and whose last $q-j$ columns are all zero vectors, i.e., $[\bC]_j = [\bc_1 \; \bc_2 \cdots\bc_j \; \bzero \cdots \bzero]$. All proofs are provided in the appendix. 
\begin{theorem} \label{gendim}  
Any global minimum $(\bW,\bH,\bb)$ for (\ref{formofloss}) takes the following form:
If $0 < c < \lambda_{\max}$, then for any semi-orthogonal matrix $\bR$,
\begin{equation} \label{optima1}
\bW  =  \left( {\frac{\lH}{\lW}} \right)^{1/4} [\bA^{1/2}]_{j*}\bR, \hspace{.2in}
\bH  =   \sqrt{\frac{\lW}{\lH}} \bW^{T}  [\bSigma^{1/2}]^{-1} (\bY - \bYb),
\hspace{.2in} \bb =  \byb,
\end{equation}
where $j* := \max \{ j : \lambda_j \geq c \}$.
If $c > \lambda_{\max}$, then $(\bW,\bH,\bb) = (\bzero,\bzero,\byb)$.
Furthermore, if $(\bW,\bH,\bb)$ is a critical point but not a global minimum, then it is a strict saddle point. 
\end{theorem}
Theorem \ref{gendim} has numerous implications, which we elaborate on below.

\subsection{One-dimensional univariate case}
\label{sec:1d_theory}
In this subsection, we highlight the important special case $n=1$, which often arises in practice (such as with Carla 1D and the UTKface datasets). When $n=1$, $\bSigma$ is simply the scalar $\sigma^2$, which is the variance of the one-dimensional targets over the $M$ samples. 
Also, $\bW$ is a row vector, which we denote by $\bw$. 
Theorem \ref{gendim}, for $n=1$ provides the following insights:
\begin{enumerate}
    \item Depending on whether $0 < c < \sigma^2$ or not, the global minimum takes on strikingly different forms. In the case, $c > \sigma^2$, corresponding to very large regularization parameters, the optimization problem ignores the MSE and entirely focuses on minimizing the norms $||\bH||_{F}^2$ and $||\bw||_2^2$, giving $||\bH||_{F}^2 =0 $, $||\bw||_2^2 =0$.
    \item When $0<c < \sigma^2$, the optimal solution takes a more natural and interesting form:
    For any unit vector $\be \in \Rd$, the solution $(\bH, \bw, b)$ given by 
\begin{equation} \label{1doptland}
\bw^{T} = \sqrt{\lH \left(\frac{\sigma}{c^{1/2}}-1\right)} \be, \hspace{.4in}
\bH =  \frac{\sqrt{\lW}}{\sqrt{\lH} \sigma} \bw^{T} (\bY - \bYb), \hspace{.4in}
b = \yb,
\end{equation}
is a global minimum. Thus, all vectors $\bw$ on the sphere given by $||\bw||_2^2 = \lH (\frac{\sigma}{c^{1/2}}-1)$ are optimal solutions. Furthermore, $\bh_i$, $i=1,\dots,M$, are all in the one-dimensional subspace spanned by $\bw$.  Thus the optimal solution of the UFM model provides a theoretical explanation for NRC1-NRC2.
(NRC3 is not meaningful for the one-dimensional case.)
Note that the  $\bhi$'s have a global zero mean and the norm of $\bhi$ is proportional to $|y_i - \yb|$.
\end{enumerate}
\subsection{General $n$-dimensional multivariate case}
In most cases of practical interest, we will have $c < \lmin$, so that $[\bA^{1/2}]_{j*} = \bA^{1/2}$ in Theorem \ref{gendim}. 
\begin{corollary} \label{conseq}
Suppose $0 < c < \lambda_{\text{min}}$. Then the global minima given by \eqref{optima1} have the following properties:

(i) All of the $d$-dimensional feature vectors $\bhi$, $i=1,\ldots,M$, lie in the $n$-dimensional subspace spanned by the $n$ rows of $\bW$. 
(ii) $\bW \bW^{T} = \sqrt{\frac{\lH}{\lW}} \left[\bSigma^{1/2} -\sqrt{c}\bI_n \right]$, 
(iii) $\lH ||\bH||_F^2= M \lW ||\bW||_F^2$, 
(iv) $\mathcal{L}(\bH, \bW, \bb)=nc/2+\sqrt{c} ||\bA^{1/2}||_F^2$, 
(v) $\bW \bH +\bb \mathbf{1}_M^T-\bY=-\sqrt{c}[\bSigma^{1/2}]^{-1}(\bY-\bYb)$. 
\end{corollary}
From Theorem \ref{gendim} and Corollary \ref{conseq}, we make the following observations:
\begin{enumerate}
    \item Most importantly, the global minima in the UFM solution match the empirical properties (NRC1)-(NRC3) observed in Section 3. In particular, the theory precisely predicts NRC3, with $\gamma = c$. This confirms that the UFM model is an appropriate model for neural regression. 
    \item Unlike the one-dimensional case, the feature vectors are no longer colinear with any of the rows of $\bW$. Moreover, after rotation and projection (determined by the semi-orthogonal matrix $\bR$), the angles between the target vectors in $\bY - \bYb$ do not in general align with the angles between the feature vectors in $\bH$. However, if the target components are uncorrelated, so that $\bSigma$ is diagonal, then $\bA$ is also diagonal and there is alignment between $\bH$ and $\bY - \bYb$.
\end{enumerate}

Theorem \ref{gendim} also provides insight into the ``strong regularization'' case of $c > \lmin$. In this case, the rows of $\bW$ and the feature vectors $\bH$ in the global minima belong to a subspace that has dimension even smaller than $n$, specifically, to dimension $j* <n$. To gain some insight, assume that the target components are uncorrelated so that $\bSigma$ is diagonal and $\lambda_j = \sigma_j^2$, i.e., $\sigma_j^2$ is the variance of the $j$-th target component. Then for a target component for which $c > \sigma_j^2$, the corresponding row in $\bW$ will be zero and the component prediction will be $\hat{\by}_i^{(j)} = \byb^{(j)}$ for all examples $i=1,\ldots,M$. For more details, we refer the reader to Section \ref{specialcases} in the appendix.

\subsection{Removing regularization}
In the previous theorem and corollary, we assumed the presence or L2 regularization for $\bW$ and $\bH$, that is, we assumed $\lW >0$ and $\lH >0$. Now we explore the structure of the solutions to the UFM when $\lW = \lH = 0$.  In this case, the UFM model is modeling the real problem with $\lambda_{WD}$ equal to or close to zero. The loss function becomes:
\begin{equation}
L(\bW,\bH) = \frac{1}{2M} || \bW \bH - \bY ||_F^2.
\end{equation}
For this case, we do not need bias since we can obtain zero loss without it. 
\begin{theorem} \label{no_regularization}
    The solution $(\bW,\bH)$ is a global minimum if and only if $\bW$ is any $n \times d$ full rank matrix and 
    \begin{equation} \label{noreg}
        \bH = \bW^+ \bY + (\bI_d - \bW^+\bW)\bZ,
    \end{equation}
where $\bW^+$ is the pseudo-inverse of $\bW$ and $\bZ$ is any $d \times M$ matrix.
Consequently, when there is no regularization, for each full-rank $\bW$
there is an infinite number of global minima $(\bW,\bH)$ that do not collapse to any subspace of $\Rd$.
\end{theorem}
From Theorem \ref{no_regularization}, when there is no regularization, the feature vectors do not collapse. Moreover,  
any full rank $\bW$ provides an optimal solution. For example, for $n=2$, the two rows of $\bW$ can have any angle between them except angle 0 and angle 180. 
This is very different from the results we have for $\lH, \lW > 0$, in which case $\bW$ depends on the covariance matrix $\bSigma$.
Note that if we set $\lH =\lW$ and let $\lH \rightarrow 0$, then the limit of $\bW$ still depends on $\bSigma$.
Thus there is a major discontinuity in the solution when $\lH, \lW$ goes to zero. We also observed this phase shift in the experiments (see Figure \ref{small_reg}).
We can therefore conclude that neural regression collapse is not an intrinsic property of neural regression alone. The geometric structure of neural regression collapse is due to the inclusion of regularization in the loss function.

\subsection{Empirical results with UFM assumptions}
\label{sec:case2_exp}
We also provide empirical results for the case when we train with the same form of regularization as assumed by the UFM model. Specifically, we turn off weight decay and add an L2 penalty on the last-layer features $\bhi$, $i=1,\ldots,M$, and on the layer linear weights $\bW$.  Additionally, we omit the ReLU activation function in the penultimate layer, allowing the feature representation produced by the feature extractor to take any value, thus reflecting the UFM model.
For these empirical results, when evaluating NRC3, rather than searching for $\gamma$ as in the definition of NRC3, 
we use the exact value of $\gamma$ given by Theorem \ref{gendim}, that is,  $\gamma = \lW \lH = c$. 


\begin{figure*}[h]
    \centering
    \includegraphics[width=1.0\linewidth]{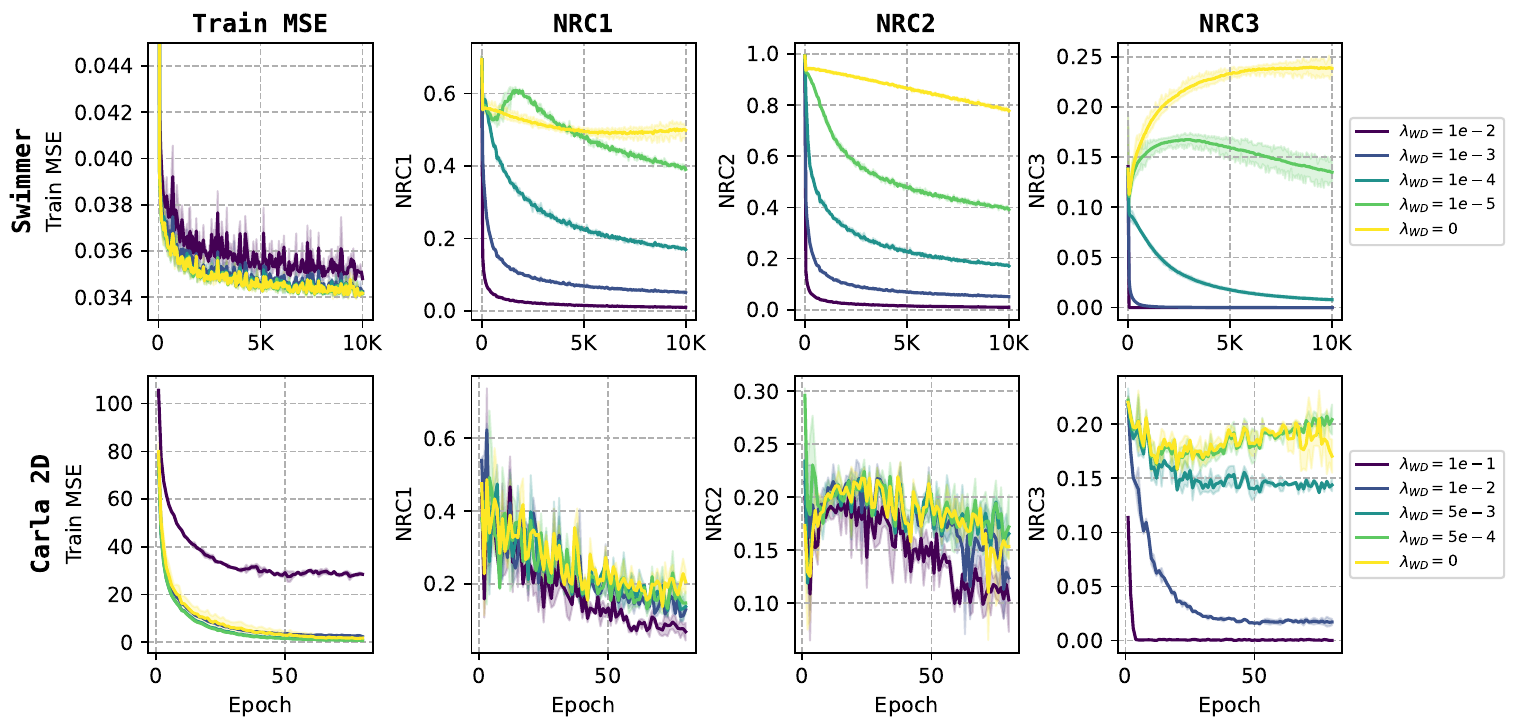}
    \caption{Empirical results with UFM assumption where L2 regularization on $\bH$ and $\bW$ are used instead of weight decay.}
    \label{fig:nc_case2}
\end{figure*}

Figure \ref{fig:nc_case2} illustrates training MSE and NRC metrics for varying values of $c$. 
(For simplicity, we only considered the case where $\lW = \lH$. Appendix \ref{sec:b_exp} contains results on remaining datasets.)
As we are considering a different model and loss function for these empirical experiments, convergence occurs more quickly and so we train for a smaller number of epochs. 
We can conclude that the UFM theory not only accurately predicts the behavior of the standard L2 regularization approach with weight-decay for all parameters (Figure \ref{small_reg}), 
but also accurately predicts the behavior when regularization follows the UFM assumptions (Figure \ref{fig:nc_case2}). 

 \section{Conclusion}
 \label{sec5}
 We provided strong evidence, both empirically and theoretically, of the existence of neural collapse for multivariate regression. 
 This extension is significant not only because it broadens the applicability of neural collapse to a new category of problems but also because it suggests that the phenomena of neural collapse could be a universal behavior in deep learning.  However, it is worth acknowledging that while we have gained a better understanding of the model behavior of deep regression models in the terminal phase of training, we have not addressed the connection between neural regression collapse and model generalization. This crucial aspect remains an important topic for future research.

\clearpage
\begin{ack}
    This work is submitted in part by the NYU Abu Dhabi Center for Artificial Intelligence and Robotics, funded by Tamkeen under the Research Institute Award CG010.
    
    This work is partially supported by Shanghai Frontiers Science Center of Artificial Intelligence and Deep Learning at NYU Shanghai. Experimental computation was supported in part through the NYU IT High-Performance Computing resources and services.
\end{ack}

\bibliography{reference}{}

\begin{thebibliography}{38}
\providecommand{\natexlab}[1]{#1}
\providecommand{\url}[1]{\texttt{#1}}
\expandafter\ifx\csname urlstyle\endcsname\relax
  \providecommand{\doi}[1]{doi: #1}\else
  \providecommand{\doi}{doi: \begingroup \urlstyle{rm}\Url}\fi

\bibitem[Brockman et~al.(2016)Brockman, Cheung, Pettersson, Schneider, Schulman, Tang, and Zaremba]{brockman2016openai}
Greg Brockman, Vicki Cheung, Ludwig Pettersson, Jonas Schneider, John Schulman, Jie Tang, and Wojciech Zaremba.
\newblock Openai gym, 2016.

\bibitem[Codevilla et~al.(2018)Codevilla, M{\"u}ller, L{\'o}pez, Koltun, and Dosovitskiy]{Codevilla2018}
Felipe Codevilla, Matthias M{\"u}ller, Antonio L{\'o}pez, Vladlen Koltun, and Alexey Dosovitskiy.
\newblock End-to-end driving via conditional imitation learning.
\newblock In \emph{International Conference on Robotics and Automation (ICRA)}, 2018.

\bibitem[Dang et~al.(2023)Dang, Nguyen, Tran, Tran, and Ho]{dang2023neural}
Hien Dang, Tan Nguyen, Tho Tran, Hung Tran, and Nhat Ho.
\newblock Neural collapse in deep linear network: From balanced to imbalanced data.
\newblock \emph{arXiv preprint arXiv:2301.00437}, 2023.

\bibitem[Ergen and Pilanci(2021)]{ergen2021revealing}
Tolga Ergen and Mert Pilanci.
\newblock Revealing the structure of deep neural networks via convex duality.
\newblock In \emph{International Conference on Machine Learning}, pages 3004--3014. PMLR, 2021.

\bibitem[Fang et~al.(2021)Fang, He, Long, and Su]{fang2021exploring}
Cong Fang, Hangfeng He, Qi~Long, and Weijie~J Su.
\newblock Exploring deep neural networks via layer-peeled model: Minority collapse in imbalanced training.
\newblock \emph{Proceedings of the National Academy of Sciences}, 118\penalty0 (43):\penalty0 e2103091118, 2021.

\bibitem[Fisher et~al.(2024)Fisher, Meng, and Papyan]{fisher2024pushing}
Quinn Fisher, Haoming Meng, and Vardan Papyan.
\newblock Pushing boundaries: Mixup's influence on neural collapse.
\newblock \emph{arXiv preprint arXiv:2402.06171}, 2024.

\bibitem[Fu et~al.(2020)Fu, Kumar, Nachum, Tucker, and Levine]{fu2020d4rl}
Justin Fu, Aviral Kumar, Ofir Nachum, George Tucker, and Sergey Levine.
\newblock {D4RL}: Datasets for deep data-driven reinforcement learning.
\newblock \emph{arXiv preprint arXiv:2004.07219}, 2020.

\bibitem[Galanti et~al.(2021)Galanti, Gy{\"o}rgy, and Hutter]{galanti2021role}
Tomer Galanti, Andr{\'a}s Gy{\"o}rgy, and Marcus Hutter.
\newblock On the role of neural collapse in transfer learning.
\newblock \emph{arXiv preprint arXiv:2112.15121}, 2021.

\bibitem[Gallouédec et~al.(2024)Gallouédec, Beeching, Romac, and Dellandréa]{gallouedec2024jack}
Quentin Gallouédec, Edward Beeching, Clément Romac, and Emmanuel Dellandréa.
\newblock {Jack of All Trades, Master of Some, a Multi-Purpose Transformer Agent}.
\newblock \emph{arXiv preprint arXiv:2402.09844}, 2024.
\newblock URL \url{https://arxiv.org/abs/2402.09844}.

\bibitem[Guo et~al.(2024)Guo, Ross, Zhao, George, Ling, Xu, and Dong]{guo2024cross}
Li~Guo, Keith Ross, Zifan Zhao, Andriopoulos George, Shuyang Ling, Yufeng Xu, and Zixuan Dong.
\newblock Cross entropy versus label smoothing: A neural collapse perspective.
\newblock \emph{arXiv preprint arXiv:2402.03979}, 2024.

\bibitem[Han et~al.(2021)Han, Papyan, and Donoho]{han2021neural}
XY~Han, Vardan Papyan, and David~L Donoho.
\newblock Neural collapse under {MSE} loss: Proximity to and dynamics on the central path.
\newblock \emph{arXiv preprint arXiv:2106.02073}, 2021.

\bibitem[He and Su(2023)]{he2023law}
Hangfeng He and Weijie~J Su.
\newblock A law of data separation in deep learning.
\newblock \emph{Proceedings of the National Academy of Sciences}, 120\penalty0 (36):\penalty0 e2221704120, 2023.

\bibitem[He et~al.(2016)He, Zhang, Ren, and Sun]{he2016deep}
Kaiming He, Xiangyu Zhang, Shaoqing Ren, and Jian Sun.
\newblock Deep residual learning for image recognition.
\newblock In \emph{Proceedings of the IEEE conference on computer vision and pattern recognition}, pages 770--778, 2016.

\bibitem[Hong and Ling(2023)]{hong2023neural}
Wanli Hong and Shuyang Ling.
\newblock Neural collapse for unconstrained feature model under cross-entropy loss with imbalanced data.
\newblock \emph{arXiv preprint arXiv:2309.09725}, 2023.

\bibitem[Hornik et~al.(1989)Hornik, Stinchcombe, and White]{hornik1989multilayer}
Kurt Hornik, Maxwell Stinchcombe, and Halbert White.
\newblock Multilayer feedforward networks are universal approximators.
\newblock \emph{Neural networks}, 2\penalty0 (5):\penalty0 359--366, 1989.

\bibitem[Hui et~al.(2022)Hui, Belkin, and Nakkiran]{hui2022limitations}
Like Hui, Mikhail Belkin, and Preetum Nakkiran.
\newblock Limitations of neural collapse for understanding generalization in deep learning.
\newblock \emph{arXiv preprint arXiv:2202.08384}, 2022.

\bibitem[Kim and Kim()]{kimfixed}
Hoyong Kim and Kangil Kim.
\newblock Fixed non-negative orthogonal classifier: Inducing zero-mean neural collapse with fea-ture dimension separation.

\bibitem[Kothapalli(2022)]{kothapalli2022neural}
Vignesh Kothapalli.
\newblock Neural collapse: A review on modelling principles and generalization.
\newblock \emph{arXiv preprint arXiv:2206.04041}, 2022.

\bibitem[Mixon et~al.(2020)Mixon, Parshall, and Pi]{mixon2020neural}
Dustin~G Mixon, Hans Parshall, and Jianzong Pi.
\newblock Neural collapse with unconstrained features.
\newblock \emph{arXiv preprint arXiv:2011.11619}, 2020.

\bibitem[Papyan et~al.(2020)Papyan, Han, and Donoho]{papyan2020prevalence}
Vardan Papyan, XY~Han, and David~L Donoho.
\newblock Prevalence of neural collapse during the terminal phase of deep learning training.
\newblock \emph{Proceedings of the National Academy of Sciences}, 117\penalty0 (40):\penalty0 24652--24663, 2020.

\bibitem[Poggio and Liao(2020)]{poggio2020explicit}
Tomaso Poggio and Qianli Liao.
\newblock Explicit regularization and implicit bias in deep network classifiers trained with the square loss.
\newblock \emph{arXiv preprint arXiv:2101.00072}, 2020.

\bibitem[Rangamani et~al.(2023)Rangamani, Lindegaard, Galanti, and Poggio]{rangamani2023feature}
Akshay Rangamani, Marius Lindegaard, Tomer Galanti, and Tomaso~A Poggio.
\newblock Feature learning in deep classifiers through intermediate neural collapse.
\newblock In \emph{International Conference on Machine Learning}, pages 28729--28745. PMLR, 2023.

\bibitem[Schulman et~al.(2017)Schulman, Wolski, Dhariwal, Radford, and Klimov]{schulman2017proximal}
John Schulman, Filip Wolski, Prafulla Dhariwal, Alec Radford, and Oleg Klimov.
\newblock Proximal policy optimization algorithms, 2017.

\bibitem[Seo et~al.(2024)Seo, Koh, Jeung, Lee, Kim, Lee, Cho, Choi, Kim, and Choi]{seo2024learning}
Minhyuk Seo, Hyunseo Koh, Wonje Jeung, Minjae Lee, San Kim, Hankook Lee, Sungjun Cho, Sungik Choi, Hyunwoo Kim, and Jonghyun Choi.
\newblock Learning equi-angular representations for online continual learning.
\newblock \emph{arXiv preprint arXiv:2404.01628}, 2024.

\bibitem[S{\'u}ken{\'\i}k et~al.(2024)S{\'u}ken{\'\i}k, Mondelli, and Lampert]{sukenik2024deep}
Peter S{\'u}ken{\'\i}k, Marco Mondelli, and Christoph~H Lampert.
\newblock Deep neural collapse is provably optimal for the deep unconstrained features model.
\newblock \emph{Advances in Neural Information Processing Systems}, 36, 2024.

\bibitem[Sun et~al.(2017)Sun, Shrivastava, Singh, and Gupta]{sun2017revisiting}
Chen Sun, Abhinav Shrivastava, Saurabh Singh, and Abhinav Gupta.
\newblock Revisiting unreasonable effectiveness of data in deep learning era.
\newblock In \emph{Proceedings of the IEEE international conference on computer vision}, pages 843--852, 2017.

\bibitem[Tarasov et~al.(2022)Tarasov, Nikulin, Akimov, Kurenkov, and Kolesnikov]{tarasov2022corl}
Denis Tarasov, Alexander Nikulin, Dmitry Akimov, Vladislav Kurenkov, and Sergey Kolesnikov.
\newblock {CORL}: Research-oriented deep offline reinforcement learning library.
\newblock In \emph{3rd Offline RL Workshop: Offline RL as a ''Launchpad''}, 2022.
\newblock URL \url{https://openreview.net/forum?id=SyAS49bBcv}.

\bibitem[Thrampoulidis et~al.(2022)Thrampoulidis, Kini, Vakilian, and Behnia]{thrampoulidis2022imbalance}
Christos Thrampoulidis, Ganesh~Ramachandra Kini, Vala Vakilian, and Tina Behnia.
\newblock Imbalance trouble: Revisiting neural-collapse geometry.
\newblock \emph{Advances in Neural Information Processing Systems}, 35:\penalty0 27225--27238, 2022.

\bibitem[Tirer and Bruna(2022)]{tirer2022extended}
Tom Tirer and Joan Bruna.
\newblock Extended unconstrained features model for exploring deep neural collapse.
\newblock In \emph{International Conference on Machine Learning}, pages 21478--21505. PMLR, 2022.

\bibitem[Tirer et~al.(2023)Tirer, Huang, and Niles-Weed]{tirer2023perturbation}
Tom Tirer, Haoxiang Huang, and Jonathan Niles-Weed.
\newblock Perturbation analysis of neural collapse.
\newblock In \emph{International Conference on Machine Learning}, pages 34301--34329. PMLR, 2023.

\bibitem[Todorov et~al.(2012)Todorov, Erez, and Tassa]{mujoco}
Emanuel Todorov, Tom Erez, and Yuval Tassa.
\newblock Mujoco: A physics engine for model-based control.
\newblock In \emph{2012 IEEE/RSJ International Conference on Intelligent Robots and Systems}, pages 5026--5033, 2012.
\newblock \doi{10.1109/IROS.2012.6386109}.

\bibitem[Towers et~al.(2023)Towers, Terry, Kwiatkowski, Balis, Cola, Deleu, Goulão, Kallinteris, KG, Krimmel, Perez-Vicente, Pierré, Schulhoff, Tai, Shen, and Younis]{towers_gymnasium_2023}
Mark Towers, Jordan~K. Terry, Ariel Kwiatkowski, John~U. Balis, Gianluca~de Cola, Tristan Deleu, Manuel Goulão, Andreas Kallinteris, Arjun KG, Markus Krimmel, Rodrigo Perez-Vicente, Andrea Pierré, Sander Schulhoff, Jun~Jet Tai, Andrew Tan~Jin Shen, and Omar~G. Younis.
\newblock Gymnasium, March 2023.
\newblock URL \url{https://zenodo.org/record/8127025}.

\bibitem[Wojtowytsch et~al.(2020)]{wojtowytsch2020emergence}
Stephan Wojtowytsch et~al.
\newblock On the emergence of simplex symmetry in the final and penultimate layers of neural network classifiers.
\newblock \emph{arXiv preprint arXiv:2012.05420}, 2020.

\bibitem[Yang et~al.(2022)Yang, Chen, Li, Xie, Lin, and Tao]{yang2022inducing}
Yibo Yang, Shixiang Chen, Xiangtai Li, Liang Xie, Zhouchen Lin, and Dacheng Tao.
\newblock Inducing neural collapse in imbalanced learning: Do we really need a learnable classifier at the end of deep neural network?
\newblock \emph{Advances in neural information processing systems}, 35:\penalty0 37991--38002, 2022.

\bibitem[Zhang et~al.(2017)Zhang, Song, and Qi]{zhifei2017cvpr}
Zhifei Zhang, Yang Song, and Hairong Qi.
\newblock Age progression/regression by conditional adversarial autoencoder.
\newblock In \emph{IEEE Conference on Computer Vision and Pattern Recognition (CVPR)}. IEEE, 2017.

\bibitem[Zhou et~al.(2022{\natexlab{a}})Zhou, Li, Ding, You, Qu, and Zhu]{zhou2022optimization}
Jinxin Zhou, Xiao Li, Tianyu Ding, Chong You, Qing Qu, and Zhihui Zhu.
\newblock On the optimization landscape of neural collapse under {MSE} loss: Global optimality with unconstrained features.
\newblock In \emph{International Conference on Machine Learning}, pages 27179--27202. PMLR, 2022{\natexlab{a}}.

\bibitem[Zhou et~al.(2022{\natexlab{b}})Zhou, You, Li, Liu, Liu, Qu, and Zhu]{zhou2022all}
Jinxin Zhou, Chong You, Xiao Li, Kangning Liu, Sheng Liu, Qing Qu, and Zhihui Zhu.
\newblock Are all losses created equal: A neural collapse perspective.
\newblock \emph{Advances in Neural Information Processing Systems}, 35:\penalty0 31697--31710, 2022{\natexlab{b}}.

\bibitem[Zhu et~al.(2021)Zhu, Ding, Zhou, Li, You, Sulam, and Qu]{zhu2021geometric}
Zhihui Zhu, Tianyu Ding, Jinxin Zhou, Xiao Li, Chong You, Jeremias Sulam, and Qing Qu.
\newblock A geometric analysis of neural collapse with unconstrained features.
\newblock \emph{Advances in Neural Information Processing Systems}, 34:\penalty0 29820--29834, 2021.

\end{thebibliography}
\bibliographystyle{plainnat}

\clearpage
\appendix

\section{Experimental details for Section \ref{sec:exp_case1}}
\label{sec:a_exp}

\subsection{MuJoCo}
\label{appendix:mujoco}

For Reacher and Swimmer environments, the datasets come from an open-source repository \citep{gallouedec2024jack} and contain expert data collected by a policy trained by PPO \citep{schulman2017proximal}. The hopper dataset is part of the D4RL datasets \citep{fu2020d4rl}, a well-acknowledged benchmark for offline reinforcement learning research. Table \ref{table:mujoco_mlp} summarizes all model hyperparameters and experimental settings used in section \ref{sec:exp_case1}. In all experiments, we train the models long enough so that the model weights converge. We provide more details below about the MuJoCo datasets employed and some hyperparameter settings depending on each dataset.

\begin{table}[h!]
\caption{Hyperparameter settings for experiments with weight decay on MuJoCo datasets.}
\label{table:mujoco_mlp}
\vskip 0.15in
\begin{center}
\begin{small}
\begin{tabular}{cll}
\toprule
    & \textbf{Hyperparameter} & \textbf{Value}  \\
\midrule
    & Number of hidden layers   & $3$ \\
  Model Architecture  & Hidden layer dimension & $256$  \\ 
    & Activation function & ReLU  \\ 
    & Number of linear projection layer ($\bW$) & 1 \\
\midrule
    & Epochs & 1.5e6, Reacher \\
    &   & 1e6, Swimmer \\
    &   & 2e5, Hopper\\
    & Batch size & 256 \\
    & Optimizer & SGD \\
    & Learning rate & 1e-2 \\
Training    & Weight decay & 1.5e-3, Reacher \\
    &   & 1e-2, Swimmer \\
    &   & 1e-2, Hopper \\
    & Seeds & 0, 1, 2 \\
    & Compute resources & Intel(R) Xeon(R) Platinum 8268 CPU \\
    & Number of CPU compute workers & 4 \\
    & Requested compute memory & 16 GB \\
    & Approximate average execution time & 16 hours\\
    
\bottomrule
\end{tabular}
\end{small}
\end{center}
\end{table}

\paragraph{MuJoCo environment descriptions} We use expert data obtained from \citet{gallouedec2024jack} and \citet{fu2020d4rl} for the Reacher, Swimmer, and Hopper environments. Reacher is a robot arm with two joints; the goal of this environment is to control the tip of this arm to reach a randomly generated target point in a 2-dimensional plane. Swimmer is a linear-chain-like robot with three different body parts connected by two rotors; the goal of Swimmer is to move forward on a 2-dimensional plane as fast as possible. Similarly, Hopper is a 2-dimensional one-legged robot with four body parts, and the goal is to hop forward as fast as possible. All three simulated robots are controlled by applying torques on the joints connecting the body parts. Those torques are therefore the actions. In creating the datasets, online reinforcement learning was used to find expert policies \citep{gallouedec2024jack,fu2020d4rl}. To generate the offline expert datasets, the expert policy is then applied to the environment to generate episodes consisting of states $\bx_i$ and actions (that is, targets) $\byi$. The state $\bx_i$ includes robot positions, and angle, velocity, and angular velocity of all robot joints, and the targets $\byi$ include the torques on joints. 

\paragraph{Low data regime} Using regularized regression to train a neural network with expert state-action data is often referred to as {\em imitation learning}. In this paper, we follow the common practice of using relatively small MLP architectures for the MuJoCo environments \citep{tarasov2022corl}. In imitation learning, it is desirable to learn a good policy with as little expert data as possible. We therefore train the models with subsets of the expert data in the datasets for each of the three environments. Specifically, we use 20, 1, and 10 episodes (complete expert demonstrations) for Reacher, Swimmer, and Hopper, respectively. This corresponds to 1,000, 1,000, and 10,000 data points for the three environments, respectively. For each environment, we also take a subset of the full validation (test) dataset and keep the number of data $20\%$ of training data size. As we are using fewer full expert demonstrations for Swimmer, we increase the weight decay value to further mitigate overfitting in this case.


\subsection{CARLA and UTKface}
The Carla dataset is collected by recording surroundings via automotive cameras, while a human driver operates a vehicle in a simulative urban environment \citep{Codevilla2018}. The recorded images are states $\bxi$ of the vehicle and the expert control from the driver, which includes speed and steering angles, serves as actions $\by_i \in[0, 85]\times  [-1, 1]$ in the dataset. A well-trained model on this dataset is expected to drive the vehicle safely in the virtual environment. The UTKface dataset consists of full-face photographs $\bxi$ of humans whose ages range from 1 to 116 \citep{zhifei2017cvpr}. The goal of this dataset is to accurately predict the age $\byi$ of the person in each photo.

In both cases, ResNet network \citep{he2016deep} is employed as the model backbone to extract image features. And the full dataset is used for training both models, as learning a good feature extractor from visual inputs requires a large number of images \citep{he2016deep, sun2017revisiting}. To adapt ResNet architecture, a native of classification tasks, to regression tasks, we replace the last layer classifier with a fully connected layer to map learned features to the continuous targets. Depending on the task complexity, we select ResNet18 for Carla and ResNet34 for UTKface. The experimental setup for CARLA 1D/2D and UTKface datasets are summarized in Table \ref{table:carla_resnet18}. 

\begin{table}[h!]
\caption{Hyperparameters of ResNet for Carla and UTKface datasets.}
\label{table:carla_resnet18}
\vskip 0.15in
\begin{center}
\begin{small}
\begin{tabular}{cll}
\toprule
    & \textbf{Hyperparameter} & \textbf{Value}  \\
\midrule
    & Backbone of hidden layers & ResNet18, Carla \\
  Architecture    & & ResNet34, UTKface \\
    & Last layer hidden dim & $512$  \\ 
    & Final layer activation function & ReLU  \\ 
\midrule
    & Epochs & 100 \\
    & Batch size & 512 \\
    & Optimizer & SGD \\
    & Momentum & 0.9 \\
    & Learning rate & 0.001 \\
Training     & Multistep\_gamma & 0.1 \\
    & Seeds & 0, 1\\
    & Compute resources & NVIDIA A100 8358 80GB \\
    & Number of compute workers & 8 \\
    & Requested compute memory & 200 GB \\
    & Approximate average execution time & 42 hours\\
    
\bottomrule
\end{tabular}
\end{small}
\end{center}
\end{table} 

\subsection{Computing NRC3}
\label{appendix: nrc3}
For univariate regression, note that $\bW\bW^T=||\bw||_2^2$, and $\bSigma^{1/2} - \gamma^{1/2} \bI_n = \sigma - \gamma^{1/2}$, where $\bw$ is a vector of the final linear layer of the model; $\sigma$ is the standard deviation of the one-dimensional targets. Thus, NRC3 is trivially zero. Alternatively, to align with the theory in Section \ref{sec:1d_theory}, one may define one-dimensional NRC3 as:
\[
\text{NRC3} = \left|||\bw||_2^2 - \gamma_2(\sigma - \gamma_1^{1/2}) \right|^2 \to 0,
\]
for some $\gamma_1 \in (0, \lambda_{\min})$ and $\gamma_2 > 0$. However, this is also trivially true, e.g. for any $\gamma_1 \in (0, \sigma^2)$ (note that $\lambda_{\min}=\sigma^2$) and $\gamma_2 = ||\bw||_2^2(\sigma - \gamma_1^{1/2})^{-1}$ after parameters $\bw$ become stable. Therefore, we found NRC3 for univariate regression to be not as meaningful, and therefore omitted the corresponding plots.

For multivariate regression, we run all experiments long enough in order to ensure that the training has entered the terminal phase of training as measured by $R^2$ (see Figure \ref{NRC_Case1}). After training, we extract the $\bW$ matrix and identify $\gamma$ that minimizes NRC3 for that specific $\bW$. This $\gamma$ was then used to compute the NRC3 metric for all $\bW$ matrices during training, resulting in the NRC3 curves shown in Figure \ref{NRC_Case1}. Figure \ref{optimal_c} visualizes NRC3 as a function of $\gamma$ for the final trained $\bW$.

In Appendix \ref{appendix:unique_gamma}, we show that under a condition that is satisfied if $\lambda_{WD}$ is reasonably large, a non-normalized version of NRC3, see \eqref{nrc3gamma}, is convex and it has a unique minimum. Since we employ relatively large weight decay for experiments in Figure \ref{NRC_Case1}, the condition of Theorem \ref{thm:unique_gamma} is satisfied, and thus Figure \ref{optimal_c} displays a unique optimal $\gamma$ for all datasets.

\subsection{Results for small weight decay}
\label{appendix: full_fig4}
Figure \ref{fig:fig4_full_mujoco} and Figure \ref{fig:fig4_full_carla} include results on studying small weight decay values for all datasets. When weight decay approaches zero, NRC1-3 typically become larger, compared with NRC1-3 obtained with larger weight decay values.

Particularly, when there is no weight decay, we observe that NRC1-3 has a strong tendency to converge (There is a relatively small amount of collapse since gradient descent tends to seek solutions with small norms.), while the test MSE increases on small MuJoCo datasets. Theorem \ref{no_regularization} provides some insight: when there is no regularization, there is an infinite number of non-collapsed optimal solutions under UFM; whereas Theorem \ref{gendim} shows that when there is regularization, all solutions are collapsed. When there is regularization, we are seeking a small norm optimal solution, which leads to NRC1-3.

\begin{figure}[hb]
    \centering
    \includegraphics[width=1\linewidth]{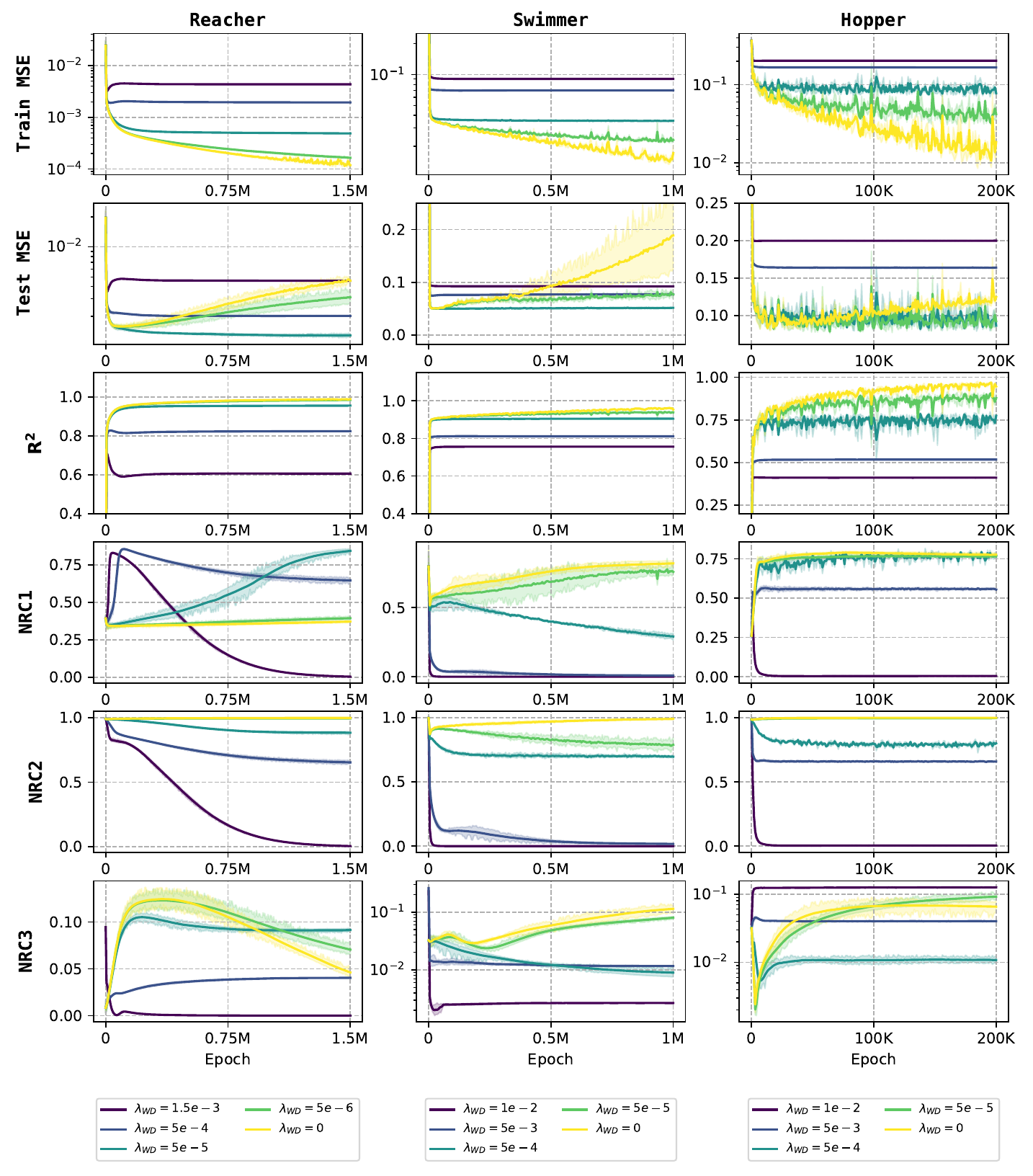}
    \caption{Train/Test MSE, $R^2$, and NRC1-3 under different weight decays for MuJoCo datasets.}
    \label{fig:fig4_full_mujoco}
\end{figure}

\begin{figure}
    \centering
    \includegraphics[width=1\linewidth]{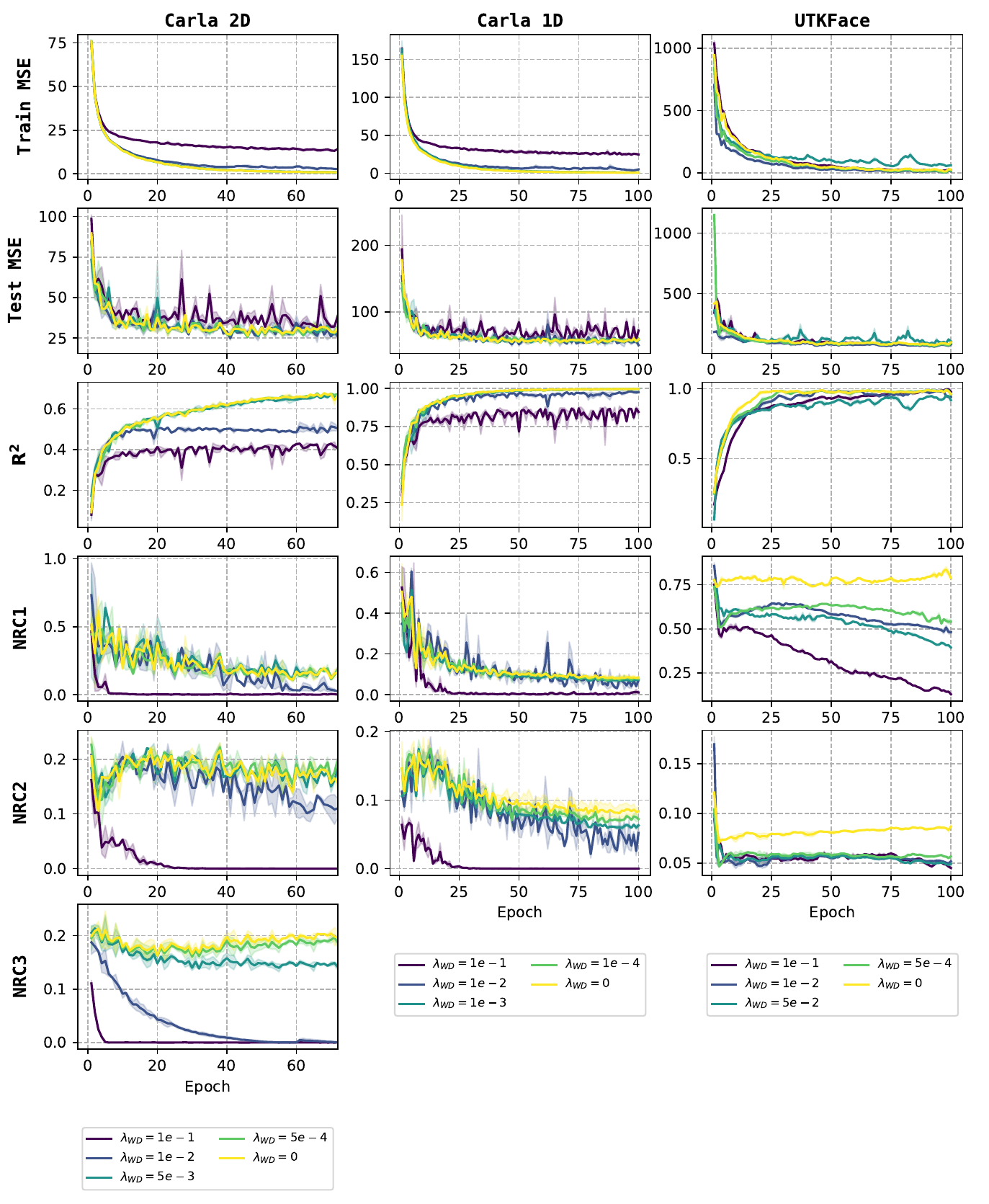}
    \caption{Train/Test MSE, $R^2$, and NRC1-3 under different weight decays for Carla and UTKFace datasets.}
    \label{fig:fig4_full_carla}
\end{figure}

\clearpage
\section{Additional experimental results}\label{sec:b_exp}

In this section, we delve into additional experiments aimed at further exploring the phenomena of neural regression collapse.

\paragraph{Complete experiments under UFM assumption} As studied in Section \ref{sec:case2_exp}, we run experiments that align with the UFM assumption and verify the theoretical NRC1-3. In addition to the L2 regularization on both $\bH$ and $\bW$, the model thus is empowered with more expressive learned feature $\bH$, e.g. removing ReLU in the penultimate layer to allow negative feature values and incorporating more training data. Complete results for all six datasets are shown in Figure \ref{fig:case2_full}. As we can see, NRC1-3 do not converge to very low values as is when regularization is stronger. This confirms our NRC theory in Section \ref{sec: theoretical} and also corresponds to the results observed in Figure \ref{fig:fig4_full_mujoco} and Figure \ref{fig:fig4_full_carla} where we apply weight decay to all model parameters in practice.

\begin{figure*}[h]
    \centering
    \includegraphics[width=1.0\linewidth]{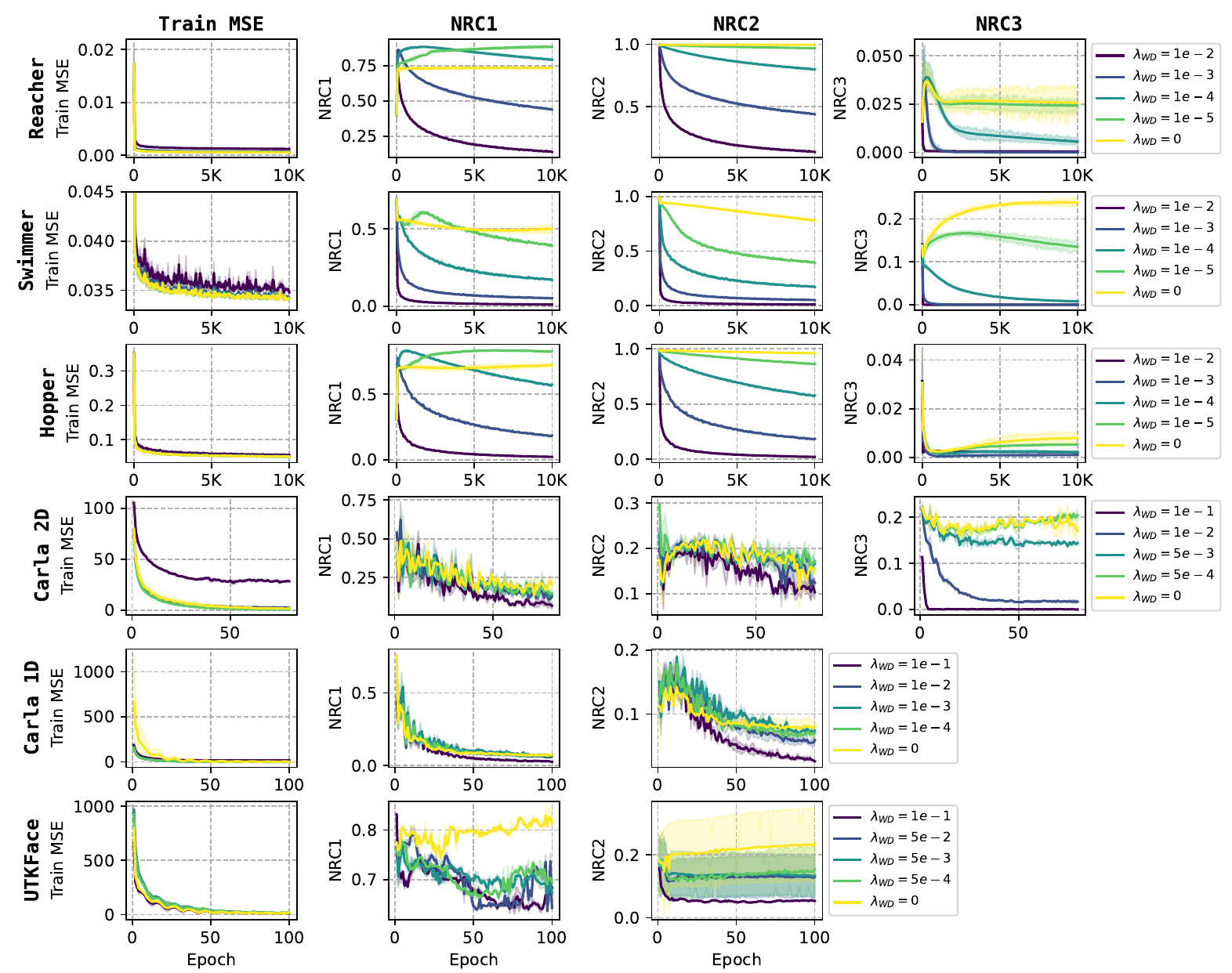}
    \caption{Empirical results with UFM assumption where L2 regularization on $\bH$ and $\bW$ are used instead of weight decay for all six datasets.}
    \label{fig:case2_full}
\end{figure*}

\paragraph{Norms of $\bH$ and $\bW$} As demonstrated in Corollary \ref{conseq}(iii), the norms of the last layer weight matrix and the feature matrix depend on the ratio of regularization parameters $\lH/\lW$. In Figure \ref{fig:wh_norm}, we empirically demonstrate how the regularization parameters impact the norms of the last layer weight matrix and the feature matrix. Specifically, we fixed $\lW=0.01$ and varied the value of $\lH$. We observe that with increasing $\lH$, the feature norm monotonically decreases, and the norms of the rows of the weight matrix monotonically increase, which is consistent with our theoretical result.

\begin{figure*}[h]
    \centering
    \includegraphics[width=0.85\linewidth]{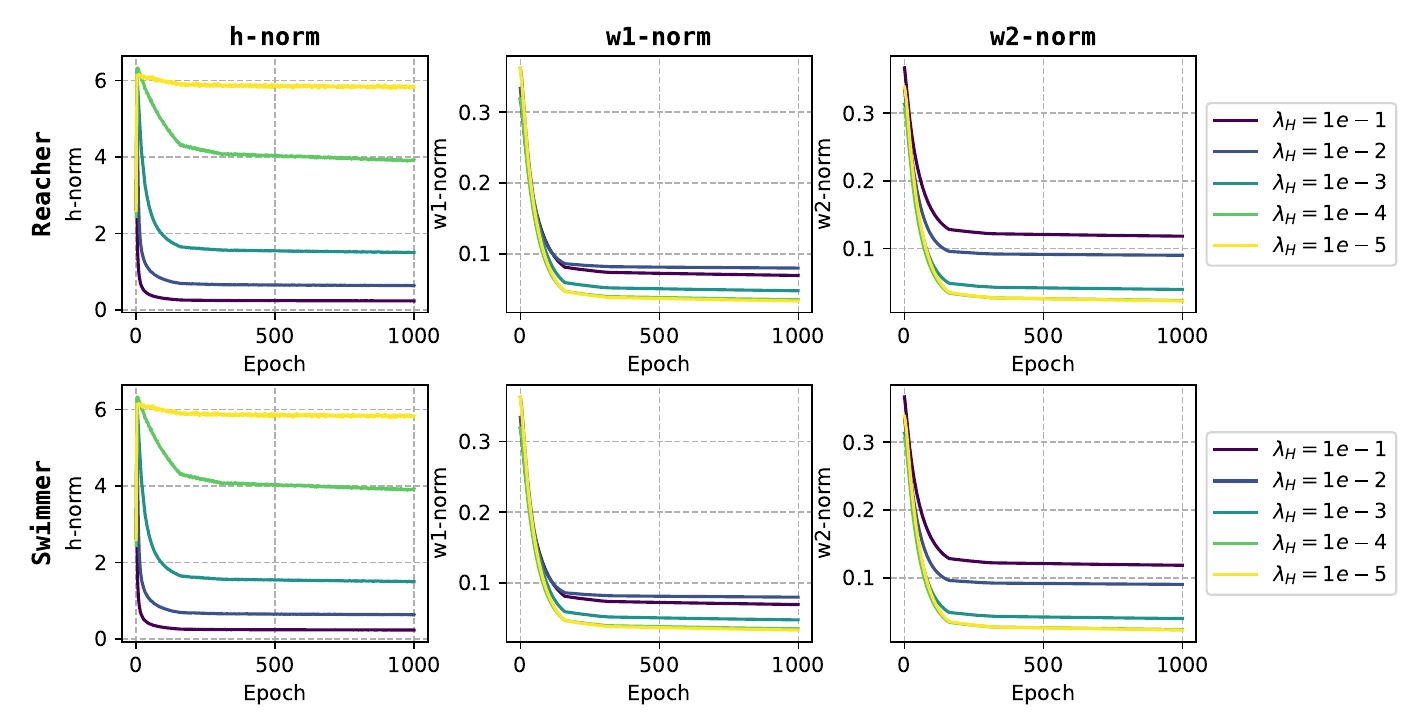}
    \caption{Comparison of the norms of $\mathbf{W}$ and $\mathbf{H}$ with fixed $\lW$ and varying $\lH$. The columns from left to right represent the model's average feature norm and the norms for $\mathbf{w}_1$ and $\mathbf{w}_2$, respectively.}
    \label{fig:wh_norm}
\end{figure*}

\begin{figure*}[h]
    \centering
    \includegraphics[width=0.9\linewidth]{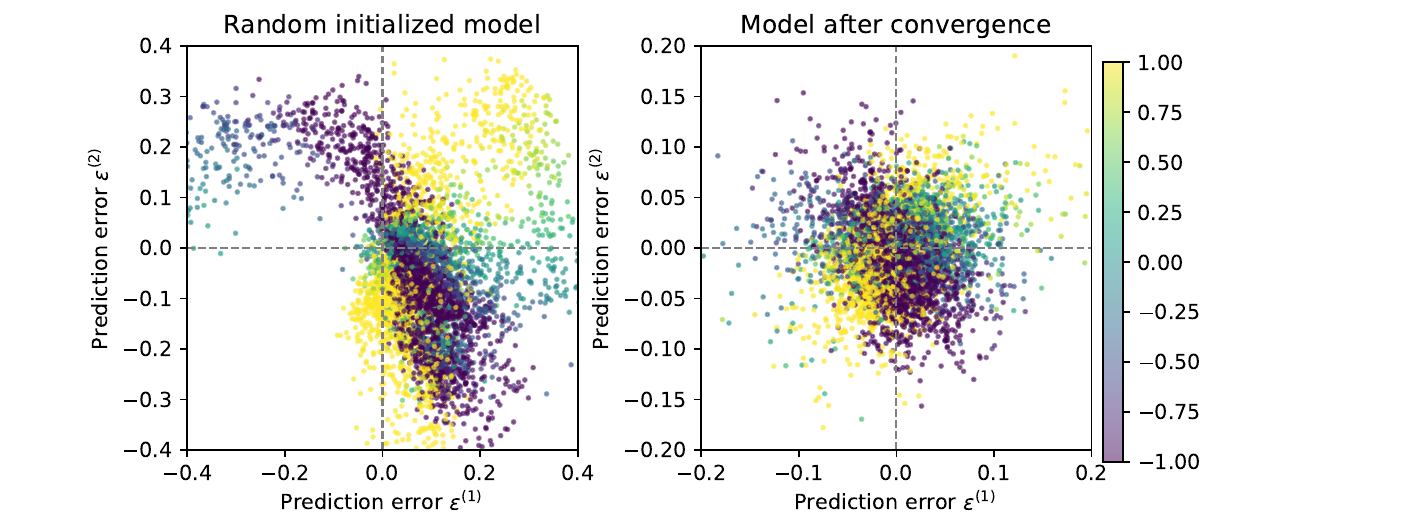}
    \caption{Residual errors $\varepsilon^{(2)}$ versus $\varepsilon^{(1)}$ for both the randomly initialized model and the trained model after convergence on the Reacher dataset. The color of the points indicates the ratio ${z}^{(2)}/{z}^{(1)}$. 
    }
    \label{fig:eps}
\end{figure*}

\paragraph{Connection to whitening} In statistical analysis, whitening (or sphering) refers to a common preprocessing step to transform random variables to orthogonality. A whitening transformation (or sphering transformation) is a linear transformation that transforms a set of vectors of random variables with a known covariance matrix into a new set of vectors of random variables such that the components of the transformed vectors are uncorrelated and have unit variances. The transformation is called ``whitening'' because it changes the input vector to white noise. Due to rotational freedom, there are infinitely many possible whitening methods, and consequently there is a diverse range of whitening procedures in use such as PCA whitening, Cholesky whitening, and Mahalonobis or zero-phase component analysis (ZCA) whitening among others. In the latter, the matrix used for the procedure of whitening is $\bW^{ZCA}=[\bSigma^{1/2}]^{-1}$, where $\bSigma$ is the covariance matrix of the original data. Interestingly, $\bW^{ZCA}$ is obtained as the whitening transformation that minimizes the MSE between the original and the whitened data, which is appealing since in many applications, it is desirable to remove correlations with minimal additional adjustment, with the aim that the transformed data remains as similar as possible to the original data.

From Corollary \ref{conseq}(v), 
we get the residual error of the regression model, $\mathbf{E} \in \mathbb{R}^{n \times M}$, which can be formulated as 
\[
\mathbf{E}=-\sqrt{c} [\bSigma^{1/2}]^{-1} (\bY-\bYb).
\]
The residual error matrix is proportional to the ZCA-whitened centered target matrix. If we denote the ZCA-whitened target matrix by 
\[
\mathbf{Z}^{ZCA} = [\bSigma^{1/2}]^{-1} (\bY-\bYb),
\]
we have that 
\[
M^{-1} \mathbf{Z}^{ZCA} (\mathbf{Z}^{ZCA})^T=\bI_n.
\]
As a consequence, the residual error matrix can be represented as $\mathbf{E}=-\sqrt{c} \mathbf{Z}^{ZCA}$, from which a significant implication arises. After the model converges, the residual error matrix will be mean zero white noise with sample covariance matrix given by $M^{-1}\mathbf{E}\mathbf{E}^T=c \bI_n$,  i.e., the errors are uncorrelated across the $n$ target dimensions and each has variance equal to $c$. 

For any given sample,  upon examining individual dimensions, it becomes apparent that the  $j$-th dimension of the residual error, $\varepsilon^{(j)}$, is proportional to the $j$-th dimension of the standardized target variable ${z}^{(j)}$.  \looseness=-1 
We trained a 4-layer MLP model on the Reacher dataset for which the target variable is 2-dimensional to validate the above-mentioned properties. We created scatter plots of the residual error $\varepsilon^{(2)}$ versus $\varepsilon^{(1)}$ for both the randomly initialized model and the trained model after convergence. Figure \ref{fig:eps} illustrates these scatter plots, with the color of the samples indicating ${z}^{(2)}/{z}^{(1)}$. As observed from the plot, after the model converges, the residual errors reduce to white noise. Additionally, from the right plot (for the model after convergence), we observe that the plot exhibits a circular pattern where the color of the samples gradually changes as they move from one quadrant to another. This indicates the consistency between $\varepsilon^{(2)}/\varepsilon^{(1)}$ and ${z}^{(2)}/{z}^{(1)}$, which is consistent with the result in Corollary \ref{conseq}(v). 

\clearpage
\section{Supplementary lemmas}
\label{sec:c_exp}
Let us recall the form of the objective:
\begin{equation} \label{formoflossnew}
\mathcal{L}(\bH, \bW, \bb)=\frac{1}{2M} ||\bW \bH + \bb \mathbf{1}_M^{T} - \bY||_F^2 + \frac{\lH}{2M} ||\bH||_{F}^2 + \frac{\lW}{2} ||\bW||_F^2,
\end{equation}
where $\mathbf{1}_M^T=[1\cdots 1]$ and $\lH, \lW>0$  regularization constants.

In Lemma \ref{suppmat1}, we demonstrate that if $(\bH, \bW, \bb)$ is critical for \eqref{formoflossnew},  then $\bW$ can be written as a closed-form function of $\bH$ and the residual error. In an analogous way, $\bH$ can be written as a closed-form function of $\bW$ and the residual error. Furthermore, $\bb=\byb$, where $\byb$ is the mean of the targets. In addition, we provide the identity that connects the matrix norms of the two, see (iii) below.
\begin{lemma} \label{suppmat1}
i) If $(\bH, \bW, \bb)$ is a critical point of \eqref{formoflossnew}, then
\begin{align*}
    \bH &= - \lH^{-1} \bW^{T} (\bW \bH+\bYb-\bY),
    \\
    \bW &= -\frac{\lW^{-1}}{M} (\bW\bH+\bYb-\bY) \bH^{T},
    \\
    \bb &= \byb.
\end{align*}
ii) If $(\bH, \bW, \bb)$ is a critical point of \eqref{formoflossnew}, for fixed $(\bH, \bW)$, $\bb=\byb$ minimizes $\mathcal{L}(\bH, \bW, \bb)$.
\\
iii) $\lH ||\bH||_F^2=M \lW ||\bW||_F^2$.
\end{lemma}

\begin{proof}
i) 
To prove the first part of the lemma, we will proceed by equating to zero the gradients w.r.t. the variables of the optimization objective $\mathcal{L}$.
Those can be written in the form of a matrix in the following way:
\begin{align}
\frac{\partial \mathcal{L}}{\partial \bH}&=
\bW^{T} \frac{1}{M} (\bW \bH+\bb \mathbf{1}_M^{T}-\bY)+\frac{\lH}{M} \bH, \label{gradientfeat}
\\
\frac{\partial \mathcal{L}}{\partial \bW}&=
\frac{1}{M} (\bW \bH+\bb \mathbf{1}_M^{T}-\bY) \bH^{T}+\lW \bW, \label{gradientofw}
\\
\frac{\partial \mathcal{L}}{\partial \bb} &= \frac{1}{M}(\bW \bH + \bb \mathbf{1}_M^T -\bY) \mathbf{1}_M. 
\label{gradientofbias}
\end{align}

We set $\frac{\partial \mathcal{L}}{\partial \bb}=\mathbf{0}$ in \eqref{gradientofbias} and observe that 
\begin{equation} \label{formofbias}
\bb = \frac{1}{M} (\bY-\bW \bH) \mathbf{1}_M=
\frac{\bY \mathbf{1}_M}{M}-\bW \frac{\bH \mathbf{1}_M}{M}=\byb - \bW \bar{\bh}, 
\end{equation}
recalling that $\byb=M^{-1} \sum_{i=1}^M \by_i$ and $\bar{\bh}=M^{-1} \sum_{i=1}^M \bh_i$.

We set $\frac{\partial \mathcal{L}}{\partial \bW}=\mathbf{0}$ in \eqref{gradientofw} and observe that
\begin{equation} \label{gradientofw1}
    \lW \bW = -\frac{1}{M} (\bW\bH+\bb \mathbf{1}_M^T-\bY) \bH^{T}.
\end{equation}

We set $\frac{\partial \mathcal{L}}{\partial \bH}=\mathbf{0}$ in \eqref{gradientfeat} and observe that
\begin{align} \label{gradientfeat1}
\lH \bH &= - \bW^{T} (\bW \bH+\bb \mathbf{1}_M^{T}-\bY), 
\\
\lH \bh_i&=-\bW^{T} (\bW \bh_i+\bb-\by_i), \qquad \forall i=1,..,M, \label{gradientfeati}
\\
\lH \bar{\bh}&=-\bW^{T}(\bW \bar{\bh}+\bb-\byb). \label{gradientfeatavg}
\end{align}
We derived \eqref{gradientfeatavg} by summing both sides of \eqref{gradientfeati} over $i$, and subsequently dividing them by $M$. Substituting $\bb$ for $\byb-\bW \bar{\bh}$, see \eqref{formofbias}, we get
\begin{equation} \label{featmeanandb}
    \bar{\bh}=\mathbf{0}, \qquad \bb=\byb.
\end{equation}
Thus, combining \eqref{gradientofw1}, \eqref{gradientfeat1}, and \eqref{featmeanandb} completes the first part of the proof of i).

ii) If $(\bH, \bW, \bb)$ is a critical point of \eqref{formoflossnew}, noting that for fixed $(\bH, \bW)$, the objective 
$\mathcal{L}(\bH, \bW, \bb)$ is convex w.r.t. $\bb$, readily yields that $\bb=\byb$ minimizes $\mathcal{L}(\bH, \bW, \bb)$.

iii) Let us now verify that $\lH ||\bH||_F^2=M\lW ||\bW||_F^2$.
If $(\bH, \bW, \bb)$ is a critical point, then
\begin{equation} \label{normofwnh}
\frac{\partial \mathcal{L}}{\partial \bH} \bH^{T}-\bW^{T} \frac{\partial \mathcal{L}}{\partial \bW}=0.
\end{equation}
Recalling the first-order gradients of $\mathcal{L}$ w.r.t. $\bH$ and $\bW$ respectively, see \eqref{gradientfeat} and \eqref{gradientofw}, and substituting those in \eqref{normofwnh}, implies
\[
\left[\bW^{T} \frac{1}{M} (\bW \bH+\bYb-\bY)+\frac{\lH}{M} \bH \right] \bH^{T} = \bW^{T} \left[\frac{1}{M} (\bW \bH+\bYb-\bY) \bH^{T}+\lW \bW\right],
\]
which gives
\[
\frac{\lH}{M} \bH\bH^{T}=\lW \bW^{T}\bW.
\]
By definition,
\[
\frac{\lH}{M}||\bH||_F^2=\frac{\lH}{M} \text{tr}( \bH^{T} \bH)=\frac{\lH}{M} \text{tr}(\bH \bH^{T})=\lW \text{tr}(\bW^{T} \bW)=\lW ||\bW||_F^2,
\]
and this establishes the assertion $\lH ||\bH||_F^2 = M \lW ||\bW||_F^2$.

\end{proof}

Next, we touch upon various implications of Theorem \ref{gendim} in the case when $0<c<\lambda_{\text{min}}$, so that $[\bA^{1/2}]_{j*}=\bA^{1/2}$,   
where 
\begin{equation} \label{redefofa}
\bA=\bSigma^{1/2}-\sqrt{c} \bI_n.
\end{equation}  
For convenience, let us break the form of a global minimum $(\bH, \bW, \bb)$, see \eqref{optima1}, into three parts.
\begin{align}
\bW  &=   \left(\frac{\lH}{\lW}\right)^{1/4} \bA^{1/2}\bR \label{optW},
\\
\bH  &=  \left(\frac{\lW}{\lH}\right)^{1/4} \bR^{T}  \bA^{1/2} [\bSigma^{1/2}]^{-1} (\bY - \bYb), \label{optH}
\\
\bb &=  \byb, \label{optb}
\end{align}
where $\bR\in \mathbb{R}^{n\times d}$ is semi-orthogonal, i.e., $\bR \bR^T=\bI_n$. In the following lemma, we demonstrate that the residual error is a rescaled version of the centered targets, the value of the loss function at the global minimum can be computed directly by invoking the value of $c$ and the matrix norm of $\bA^{1/2}$.
\begin{lemma} \label{suppmat2}
Suppose $0<c<\lambda_{\min}$. For a global minimum $(\bH, \bW, \bb)$ of \eqref{formoflossnew}, we have that the residual error takes the following form:
\[
\bW \bH + \bYb-\bY=-\sqrt{c} [\bSigma^{1/2}]^{-1} (\bY-\bYb),
\]
Moreover, the value of the loss function $\mathcal{L}$ at the global minimum is calculated as
\[
\mathcal{L}(\bH, \bW, \bb)=\frac{n c}{2}+\sqrt{c} ||(\bSigma^{1/2}-\sqrt{c} \bI_n)^{\frac{1}{2}}||_F^2.
\]
\end{lemma}

\begin{proof}
By \eqref{optW}-\eqref{optb}, for the first assertion,
\begin{align*}
 \bW \bH+\bYb-\bY&=\bA^{1/2} \bR \bR^T \bA^{1/2} [\bSigma^{1/2}]^{-1} (\bY-\bYb) -(\bY-\bYb)
 \\
 &=\left[ \bA [\bSigma^{1/2}]^{-1} -\bI_n\right] (\bY-\bYb)
 \\
 &=-\sqrt{c} [\bSigma^{1/2}]^{-1} (\bY-\bYb),
 \end{align*}
 noting, for the first equality, that $\bR \bR^T=\bI_n$, $\bA^{1/2} \bA^{1/2} = \bA$.
 For the third equality, see \eqref{redefofa}. Therefore,
\begin{align} \label{traceoferror}
\frac{1}{M} (\bW \bH+\bYb-\bY) (\bW \bH+\bYb-\bY)^T 
=
c [\bSigma^{1/2}]^{-1} \bSigma [\bSigma^{1/2}]^{-1}=c \bI_n.
\end{align}
Using Lemma \ref{suppmat1}(iii) and \eqref{traceoferror}, we deduce
\begin{align*}
\mathcal{L}(\bH, \bW, \bb)=\frac{n c}{2}+\lW ||\bW||_F^2&=\frac{n c}{2}+\lW \sqrt{\frac{\lH}{\lW}} \text{tr}(\bA)
=\frac{n c}{2}+\sqrt{c} ||\bA^{1/2}||_F^2,
\end{align*}
which completes the proof of the lemma.

\end{proof} 

The proof of Corollary \ref{conseq} directly follows:

\subsection{Proof of corollary \ref{conseq}}

(i) is derived in the proof of Lemma \ref{suppmat1}, see \eqref{gradientfeati} and \eqref{featmeanandb}. It is also easy to derive them from the form of $\bH$ as given in \eqref{optima1}. (ii) follows by the form $\bW$, see \eqref{optima1}. 
By Lemma \ref{suppmat1}(iii) and Lemma \ref{suppmat2}, (iii)-(v) follow immediately. 

\section{Proof of theorem \ref{gendim}}
\label{sec:d_exp}
The proof of Theorem \ref{gendim} leverages \cite[Lemma B.1] {zhou2022optimization}. For clarity, we now restate their lemma in our notation.
\begin{lemma} \cite[Lemma B.1] {zhou2022optimization} \label{zhoulemma}
    For $n, d, M$ with $d\ge n$, and $\tilde{\bY}:=\bY-\bYb\in \mathbb{R}^{n\times M}$ with SVD given by 
    $\tilde{\bY}=U \tilde{\bSigma} V^T=\sum_{i=1}^n \sigma_i u_i v_i^T$, 
    where $\sigma_1\ge \sigma_2\ge \cdots \ge \sigma_n\ge 0$ are the singular values, the following problem
    \[
    \min_{\bH\in \mathbb{R}^{d\times M}, \bW \in \mathbb{R}^{n\times d}} \mathcal{L}(\bH, \bW, \byb)
    \]
    is a strict saddle function with no spurious local minima, in the sense that 
    \\
    i) Any local minimum $(\bH, \bW, \byb)$ of \eqref{formoflossnew} is a global minimum of \eqref{formoflossnew}, with the following form
    \[
    \bW \bH = U [\tilde{\bSigma}-\sqrt{M \lW \lH}\bI_n]_{+} V^T.
    \]
    Correspondingly, the minimal objective value of \eqref{formoflossnew} is 
    \[
    \mathcal{L}(\bH, \bW, \byb)= \frac{1}{2} \sum_{i=1}^n (\eta_i-\sigma_i)^2 + \sqrt{M \lW \lH} \sum_{i=1}^n \eta_i,
    \]
    where $\eta_i:=\eta_i(\lH, \lW)$ is the $i$-th diagonal entry of $[\tilde{\bSigma}-\sqrt{M \lW \lH} \bI_n]_{+}$.
    \\
    ii) Any critical point $(\bH, \bW, \byb)$ that is not a local minimum is a strict saddle point with negative curvature, i.e., the Hessian at this critical point has at least one negative eigenvalue.
\end{lemma}

Let $\tilde{\bY}=\bY-\bYb=U \tilde{\bSigma} V^T=\sum_{i=1}^n \sigma_i u_i v_i^T$, denote the compact SVD of $\tilde{\bY}\in \mathbb{R}^{n\times M}$, where $\sigma_1\ge \sigma_2\ge \cdots \ge \sigma_n > 0$ are the singular values, and $\tilde{\bSigma}\in \mathbb{R}^{n\times n}$ is diagonal, containing the aforementioned singular values. Furthermore, $U\in \mathbb{R}^{n\times n}$, $V\in \mathbb{R}^{M\times n}$ are orthogonal and semi-orthogonal respectively, i.e., $U U^T=U^TU=\bI_n$ and $V^T V=\bI_n$ respectively. For the proof, recall the value of $c=\lW \lH$.

\begin{proof}[Proof of Theorem \ref{gendim}]

Let $(\bH, \bW, \byb)$ be a global minimum of \eqref{formofloss}. By Lemma \ref{zhoulemma}, $(\bH, \bW, \byb)$ has the following form:
\begin{equation} \label{prodwh1}
\bW\bH=U [\tilde{\bSigma}-\sqrt{M c}\bI_n]_{+} V^T.
\end{equation}

In light of Lemma \ref{suppmat1} and the identity $\lH ||\bH||_F^2=M \lW ||\bW||_F^2$, from \eqref{prodwh1}, we have that 
\begin{align} 
    \bW &= \left(\frac{\lH}{M \lW}\right)^{1/4} U [\tilde{\bSigma}-\sqrt{M c} \bI_n]_{+}^{\frac{1}{2}}\bR, \label{altW}
    \\
    \bH &= \left(\frac{M \lW}{\lH}\right)^{1/4} \bR^T [\tilde{\bSigma}-\sqrt{M c} \bI_n]_{+}^{\frac{1}{2}} V^T, \label{altH}
\end{align}
for all $\bR\in \mathbb{R}^{n\times d}$ such that $\bR \bR^T=\bI_n$. Furthermore, using the SVD of $\tilde{\bY}=U \tilde{\bSigma} V^T$,
\[
\bSigma=\frac{\tilde{\bY} \tilde{\bY}^T}{M} = U \frac{\tilde{\bSigma}}{\sqrt{M}} V^T V \frac{\tilde{\bSigma}}{\sqrt{M}} U^T= U \left[\frac{\tilde{\bSigma}}{\sqrt{M}}\right]^2 U^T,
\]
which deduces $\bSigma^{1/2} = U \frac{\tilde{\bSigma}}{\sqrt{M}} U^T$. Since $U^T=U^{-1}$, this further yields   
\begin{equation} \label{eigendec}
\sqrt{M} [\bSigma^{1/2}-\sqrt{c} \bI_n]=U \left[\tilde{\bSigma} -\sqrt{M c}\bI_n\right] U^{-1},
\end{equation}
which implies that the matrices $\sqrt{M} [\bSigma^{1/2}-\sqrt{c} \bI_n]$ and $\tilde{\bSigma} -\sqrt{M c}\bI_n$ are similar. As a result, they have the same eigenvalues.
The $n\times n$ matrix on the left-hand side of \eqref{eigendec} has eigenvalues given by $\sqrt{M \lambda_i}-\sqrt{M c}$, $i=1,..., n$, where $\lambda_i$ is the $i$-th eigenvalue (in descending order) of $\bSigma$ whereas $\sigma_i-\sqrt{M c}$, $i=1,..., n$ are the (ordered) diagonal elements of the $n\times n$ matrix on the right-hand side of \eqref{eigendec}. So, 
\begin{equation} \label{eigeneq}
\sqrt{\lambda_i} = \frac{\sigma_i}{\sqrt{M}}, \qquad \text{for all } i=1,..., n.
\end{equation}

\textbf{Case I:} If $0<c <\lambda_{\text{min}}$, then by \eqref{eigeneq}, it is the case that $\sigma_i>\sqrt{M c}$, $\forall i$, and thus $[\tilde{\bSigma}-\sqrt{M c}\bI_n]^{\frac{1}{2}}_{+}=[\tilde{\bSigma}-\sqrt{M c}\bI_n]^{\frac{1}{2}}$. By \eqref{eigendec}, 
\begin{equation} \label{eigendec1}
     U [\tilde{\bSigma}-\sqrt{M c}\bI_n]^{\frac{1}{2}}=M^{1/4}[\bSigma^{1/2}-\sqrt{c}\bI_n]^{\frac{1}{2}} U,
\end{equation}
and thus \eqref{altW} becomes
\[
\bW = \left(\frac{\lH}{\lW}\right)^{1/4} [\bSigma^{1/2}-\sqrt{c} \bI_n]^{\frac{1}{2}} \tilde{\bR},
\]
for $\tilde{\bR}:=U \bR\in \mathbb{R}^{n\times d}$ semi-orthogonal. The first assertion of the theorem follows by recalling the definition of $\bA=\left[\bSigma^{1/2} -\sqrt{c}\bI_n\right]$.

For the second assertion of the theorem, it remains to observe that 
\[
[\bSigma^{1/2}]^{-1} \tilde{\bY}=\sqrt{M} U \tilde{\bSigma}^{-1} U^T U \tilde{\bSigma} V^T=\sqrt{M} U V^T.
\]
So, $V^T=M^{-1/2} U^T [\bSigma^{1/2}]^{-1} \tilde{\bY}$, and from \eqref{altH}, we get
\begin{align*}
\bH &= \left(\frac{\lW}{\lH}\right)^{1/4} \bR^T M^{-1/4} [\tilde{\bSigma}-\sqrt{M c}\bI_n]^{\frac{1}{2}} U^T [\bSigma^{1/2}]^{-1} \tilde{\bY}
\\
&= \left(\frac{\lW}{\lH}\right)^{1/4} \tilde{\bR}^T [\bSigma^{1/2}-\sqrt{c}\bI_n]^{\frac{1}{2}} [\bSigma^{1/2}]^{-1} (\bY-\bYb)
\\
&= \sqrt{\frac{\lW}{\lH}} \bW^T [\bSigma^{1/2}]^{-1} (\bY-\bYb),
\end{align*}
where the second equality follows from \eqref{eigendec1}.
\\
\textbf{Case II:} If $c>\lambda_{\text{max}}$, then by \eqref{eigeneq}, it is the case that $\sigma_i< \sqrt{M c}$, $\forall i$, and thus $[\tilde{\bSigma}-\sqrt{M c}]^{\frac{1}{2}}_{+}=\mathbf{0}$. By \eqref{altW} and \eqref{altH}, $(\bH, \bW, \byb)=(\mathbf{0}, \mathbf{0}, \byb)$ as desired.
\\
\textbf{Case III:} If $\lambda_{\text{min}}<c<\lambda_{\text{max}}$, by \eqref{eigeneq}, it is the case that
\[
[\sigma_i-\sqrt{M c}]_{+}=
\begin{cases}
    \sigma_i - \sqrt{M c}, &\text{if } i\le j*
    \\
    0, &\text{ otherwise}, 
\end{cases}
\]
where $j*=\max\{j:\lambda_j\ge c\}$, and thus
$
[\tilde{\bSigma}-\sqrt{M c}]_{+}^{\frac{1}{2}}=[\tilde{\bSigma}-\sqrt{M c}]_{j*}^{\frac{1}{2}}
$
.
By \eqref{eigendec}, 
\begin{equation} \label{eigendec2}
U \left[\tilde{\bSigma}-\sqrt{M c}\bI_n\right]^{1/2}_{j*}
=
M^{1/4} \left[\left[\bSigma^{1/2}-\sqrt{c}\bI_n\right]^{\frac{1}{2}} 
\right]_{j*}U ,
\end{equation} 
and thus \eqref{altW} becomes
\[
\bW = \left(\frac{\lH}{\lW}\right)^{1/4} \left[\left[\bSigma^{1/2}-\sqrt{c} \bI_n\right]^{\frac{1}{2}} \right]_{j*} \tilde{\bR},
\]
for $\tilde{\bR}:=U \bR\in \mathbb{R}^{n\times d}$ semi-orthogonal. The first assertion of the theorem follows by recalling the definition of $\bA=\left[\bSigma^{1/2} -\sqrt{c}\bI_n\right]$.

For the second assertion of the theorem, it remains to observe that 
\[
[\bSigma^{1/2}]^{-1} \tilde{\bY}=\sqrt{M} U \tilde{\bSigma}^{-1} U^T U \tilde{\bSigma} V^T=\sqrt{M} U V^T.
\]
So, $V^T=M^{-1/2} U^T [\bSigma^{1/2}]^{-1} \tilde{\bY}$, and from \eqref{altH}, we get
\begin{align*}
\bH &= \left(\frac{\lW}{\lH}\right)^{1/4} \bR^T M^{-1/4} \left[\tilde{\bSigma}-\sqrt{M c}\bI_n\right]_{j*}^{\frac{1}{2}} U^T [\bSigma^{1/2}]^{-1} \tilde{\bY}
\\
&= \left(\frac{\lW}{\lH}\right)^{1/4} \tilde{\bR}^T \left[\left[\bSigma^{1/2}-\sqrt{c}\bI_n\right]^{\frac{1}{2}}\right]_{j*} [\bSigma^{1/2}]^{-1} (\bY-\bYb)
\\
&= \sqrt{\frac{\lW}{\lH}} \bW^T [\bSigma^{1/2}]^{-1} (\bY-\bYb),
\end{align*}
where the second equality follows from \eqref{eigendec2}.

\end{proof}

\subsection{Examples for theorem \ref{gendim} (uncorrelated target components)} \label{specialcases}

In this subsection, we examine closely the case when $n=3$ and the target components are uncorrelated. This simplifies considerably the problem as now the covariance matrix $\bSigma$ is a diagonal matrix with entries given (in order) by $\sigma_j^2$, where $\sigma_j^2$ denotes the variance of the $j$-th target component, for $j=1,2,3$. The unique positive definite and symmetric matrix $\bA^{1/2}$, see \eqref{redefofa}, is given by 
\begin{equation} \label{rootofA}
\bA^{1/2}=
\begin{bmatrix}
(\sigma_1-\sqrt{c})^{\frac{1}{2}} & 0 & 0
\\
0 & (\sigma_2 -\sqrt{c})^{\frac{1}{2}} & 0
\\
0 & 0 & (\sigma_3 - \sqrt{c})^{\frac{1}{2}}
\end{bmatrix}.
\end{equation}
Without loss of generality, assume that $\sigma_{\text{max}}:=\sigma_1\ge \sigma_2\ge \sigma_3:=\sigma_{\text{min}}>0$. 

\begin{itemize}
    \item If $0<c < \sigma_{\text{min}}^2=\sigma_3^2$, by Theorem \ref{gendim}, $j*=3$, and therefore any global minimum $(\bH, \bW, \bb)$ of \eqref{formoflossnew} takes the following form:
    \[
    \bW = \left(\frac{\lH}{\lW}\right)^{1/4} \bA^{1/2} \bR, \qquad \bH = \sqrt{\frac{\lW}{\lH}}\bW^T [\bSigma^{1/2}]^{-1} (\bY - \bYb), \qquad \bb=\byb,
    \]
    for any semi-orthogonal matrix $\bR\in \mathbb{R}^{3\times d}$. The form of $\bA^{1/2}$, see \eqref{rootofA}, readily yields
    \[
    \bw_j^T=\sqrt{\lH \left(\frac{\sigma_j}{c^{1/2}}-1\right)} \be_j, \qquad j=1,2,3,
     \]
     c.f., \eqref{1doptland}, where $\{\be_j: j=1,2,3\}$ is any collection of vectors lying in $\mathbb{R}^d$ such that $\be_j$ is a unit vector, for all $j=1,2,3$, and $\be_k$ is orthogonal to $\be_{k'}$, for all $k\neq k'$.

     To interpret the landscape of global minima in the case when the target components are uncorrelated, the UFM ``forces'' the angle between the weight matrix rows to be $\pi/2$ (fixes the weight matrix rows to be orthogonal). Then, the configuration of the $bw_j$'s is exactly as in the 1-dimensional target case, that is those are restricted to lie on spheres of certain radiuses. The feature vector $\bh_i$ that corresponds to the $i$-th training example is then on the 3-dimensional subspace spanned by $\bw_1$, $\bw_2$ and $\bw_3$. 
     \item If $\sigma_{\text{min}}^2<c<\sigma_{\text{max}}^2$, by Theorem \ref{gendim}, $j*=1$ or $j*=2$. We analyze the latter, in which case $c<\sigma_1^2$, $c<\sigma_2^2$ but $c>\sigma_3^2$. By Theorem \ref{gendim}, any global minimum $(\bH, \bW, \bb)$ of \eqref{formoflossnew} takes the form below:
     \[
     \bW = \left(\frac{\lH}{\lW}\right)^{1/4} [\bA^{1/2}]_{2*} \bR, \qquad \bH = \sqrt{\frac{\lW}{\lH}}\bW^T [\bSigma^{1/2}]^{-1} (\bY - \bYb), \qquad \bb=\byb, 
     \]
     for any semi-orthogonal matrix $\bR\in \mathbb{R}^{3\times d}$. The form of $[\bA^{1/2}]_{2*}$, see \eqref{rootofA}, readily yields
    \begin{align*}
    \bw_j^T&=\sqrt{\lH \left(\frac{\sigma_j}{c^{1/2}}-1\right)} \be_j, \qquad j=1,2,
    \\
    \bw_3^T&=\mathbf{0},
    \end{align*}
    c.f., \eqref{1doptland}, where $\be_1, \be_2\in \mathbb{R}^d$ unit vectors orthogonal to each other. It is also worth mentioning that 
    \begin{align*}
    \bh_i &=\sqrt{\frac{\lW}{\lH}} \left[\bw_1^T \ \bw_2^T \ \mathbf{0}\right] [\bSigma^{1/2}]^{-1} (\by_i-\byb)
    \\
    &=\sqrt{\frac{\lW}{\lH}} 
    \left[\frac{(\by_i^{(1)}-\byb^{(1)})}{\sigma_1} \bw_1^T+\frac{(\by_i^{(2)}-\byb^{(2)})}{\sigma_2} \bw_2^T\right],
    \end{align*}
    for all $i=1,...,M$. In other words, for fixed $i$, the feature vector $\bh_i$ lies on $\text{span}\{\bw_1^T, \bw_2^T\}$. Moreover, the previous formula indicates that the inner product between $\bh_i$ and $\bw_j^T$ is proportional to the $j$-th standardized target component. 
    
    The analysis of the case $j*=1$, i.e., $c<\sigma_1^2$ but $c>\sigma_2^2$, $c>\sigma_3^2$ is analogous, therefore we just record the form of the $\bw_j$'s and $\bh_i$'s, omitting any further details:
    \begin{align*}
    \bw_1^T&=\sqrt{\lH \left(\frac{\sigma_1}{c^{1/2}}-1\right)} \be, \qquad \bw_2^T=\bw_3^T=\mathbf{0},
    \\
    \bh_i& = \sqrt{\frac{\lW}{\lH}} 
    \frac{(\by_i^{(1)}-\byb^{(1)})}{\sigma_1} \bw_1^T,
    \end{align*}
    for all $i=1,...,M$. In other words, for fixed $i$, the feature vector $\bh_i$ is colinear with $\bw_1^T$.
    \item If $c>\sigma_{\text{max}}^2=\sigma_1^2$, by Theorem \ref{gendim}, $(\bH, \bW, \bb)=(\mathbf{0}, \mathbf{0}, \byb)$.
\end{itemize}

\section{Proof of theorem \ref{no_regularization} (no regularization)}
\label{sec:e_exp}
We first show
\begin{equation*}
    \min_{\bW,\bH} L(\bW,\bH) = 0
\end{equation*}
Clearly $L(\bW,\bH) \geq 0$ for all $\bW$ and $\bH$. Now consider any fixed $n \times d$ matrix $\bW$ with $\text{rank}(\bW) = n$. Since $\bW$ has rank $n$, $\{ \bW \bh : \bh \in \Rd \} = \mathbb{R}^n$.
Thus there exists $\bh_i \in \Rd$ such that $\bW \bh_i = \by_i$ for all $i=1,\ldots,M$. 
Let $\bH = [\bh_1 \cdots \bh_M]$. 
For this $\bW$ and $\bH$ we have $L(\bW,\bH) = 0$. Thus $\min_{\bW,\bH} L(\bW,\bH) = 0$. 

Now let $\bW$ be any $n \times d$ matrix with full rank $n$, and consider the set of $\bH$ that satisfy $L(\bW,\bH)=0$ w.r.t. this $\bW$. This is a standard least squares problem for which the solution is well known: 
\begin{equation}
    \bH = \bW^+ \bY + (\bI_d - \bW^{+} \bW) \bZ
\end{equation}
where $\bW^+$ is the pseudo-inverse of $\bW$ and $\bZ$ is any $d \times M$ matrix.

To complete the proof, we need to show that if $\mbox{rank}(\bW) < n$, then $\bW$ cannot be part of an optimal solution for $L(\bW,\bH)$.
Suppose $\mbox{rank}(\bW) < n$. $\bW \bh = \by$ only has a solution if $\by$ lies in the column space of $\bW$. Thus $L(\bH,\bW) =0$ only if $\by_1,\ldots,\by_M$ all lie in the column space of $\bW$. Since $\mbox{rank}(\bY) = n$ and $\mbox{rank}(\bW) < n$, there will be at least one $\by_i$ that is not in column space of $\bW$. Thus for this $\by_i$ there is no $\bh_i$ such that $\bW \bh_i = \by_i$. Thus for this $\by_i$ we will have 
$(\by_i - \bW \bh_i)^2  > 0$, implying $\bW$ cannot be part of an optimal solution $\bW,\bH$.

Finally, we note from  (\ref{noreg}) that a column $\bh$ in $\bH$ is the sum of two vectors, with the first vector lying in the row space of $\bW$ and the second vector lying in the (orthogonal) null space of $\bW$. Since $\bW$ has full rank, this implies that the columns of $\bH$ can span all of $\Rd$.

\section{Proof of the uniqueness of $\gamma$}
\label{appendix:unique_gamma}
Mathematically, we can show that, under a condition that is satisfied if $\lambda_{WD}$ is reasonably large, 
\begin{equation} \label{nrc3gamma}
\text{NRC3}(\gamma):=\left|\left| \bW \bW^T-\bSigma^{1/2} + \gamma^{1/2} \bI_n\right|\right|_F^2,
\end{equation}
as given in the definition of NRC3, without normalizing the Gram matrix of $\bW$ and $\bSigma^{1/2} - \gamma^{1/2} \bI_n$, is convex and it has a unique minimum.

\begin{theorem}
\label{thm:unique_gamma}
If 
\[
\textnormal{tr}(\bSigma^{1/2})>\textnormal{tr}(\bW \bW^T),
\]
$\textnormal{NRC3}(\gamma)$, as given in \eqref{nrc3gamma}, is minimized at
\[
\gamma^*:=\left[\frac{\textnormal{tr}(\bSigma^{1/2})-\textnormal{tr}(\bW \bW^T)}{n}\right]^2.
\]
Moreover, 
\[
\textnormal{NRC3}(\gamma^*)=\left|\left|\left(\bW \bW^T-\frac{\textnormal{tr}(\bW \bW^T)}{n} \bI_n\right) - \left(\bSigma^{1/2}-\frac{\textnormal{tr}(\bSigma^{1/2})}{n} \bI_n\right)\right|\right|_F^2.
\]
\end{theorem}

\begin{proof}

Expanding the squared matrix norm appearing in the definition of $\text{NRC3}(\gamma)$, it is necessary and sufficient to minimize
\[
f(\gamma):= 2 \gamma^{1/2} \left[\text{tr}(\bW \bW^T)-\text{tr}(\bSigma^{1/2})\right] + n \gamma,
\]
which has first and second derivatives given by
\begin{align*}
f'(\gamma)&=\frac{\text{tr}(\bW \bW^T)-\text{tr}(\bSigma^{1/2})}{\gamma^{1/2}}+n,
\\
f''(\gamma)&=\frac{\text{tr}(\bSigma^{1/2})-\text{tr}(\bW \bW^T)}{2 \gamma^{3/2}}.
\end{align*}
Since 
$\text{tr}(\bSigma^{1/2})>\textnormal{tr}(\bW\bW^T)$, by the 2nd derivative test, $f$, and consequently $\text{NRC3}(\gamma)$, is convex with unique minimum achieved at 
\[
\gamma^*:=\left[\frac{\text{tr}(\bSigma^{1/2})-\text{tr}(\bW \bW^T)}{n}\right]^2.
\]

\end{proof}

\newpage
\section*{NeurIPS Paper Checklist}

\begin{enumerate}

\item {\bf Claims}
    \item[] Question: Do the main claims made in the abstract and introduction accurately reflect the paper's contributions and scope?
    \item[] Answer: \answerYes{} 
    \item[] Justification: In the abstract, we claim to empirically demonstrate the phenomenon of Neural Regression Collapse as discussed in Section \ref{sec:exp_case1}, and then we provide a theoretical explanation from the perspective of the Unconstrained Feature Model, see Section \ref{sec: theoretical} and Appendix \ref{sec:b_exp},\ref{sec:c_exp},\ref{sec:d_exp}. Thus, we show that the phenomena of neural collapse could be a universal behavior in deep learning both empirically and theoretically.


\item {\bf Limitations}
    \item[] Question: Does the paper discuss the limitations of the work performed by the authors?
    \item[] Answer: \answerYes{} 
    \item[] Justification: In Section \ref{sec5}, we point out that our explanation of neural regression collapse doesn't have implications for model generalization.
    \item[] Guidelines:
    \begin{itemize}
        \item The answer NA means that the paper has no limitation while the answer No means that the paper has limitations, but those are not discussed in the paper. 
        \item The authors are encouraged to create a separate "Limitations" section in their paper.
        \item The paper should point out any strong assumptions and how robust the results are to violations of these assumptions (e.g.,, independence assumptions, noiseless settings, model well-specification, asymptotic approximations only holding locally). The authors should reflect on how these assumptions might be violated in practice and what the implications would be.
        \item The authors should reflect on the scope of the claims made, e.g., if the approach was only tested on a few datasets or with a few runs. In general, empirical results often depend on implicit assumptions, which should be articulated.
        \item The authors should reflect on the factors that influence the performance of the approach. For example, a facial recognition algorithm may perform poorly when the image resolution is low or images are taken in low lighting. Or a speech-to-text system might not be used reliably to provide closed captions for online lectures because it fails to handle technical jargon.
        \item The authors should discuss the computational efficiency of the proposed algorithms and how they scale with dataset size.
        \item If applicable, the authors should discuss possible limitations of their approach to address problems of privacy and fairness.
        \item While the authors might fear that complete honesty about limitations might be used by reviewers as grounds for rejection, a worse outcome might be that reviewers discover limitations that aren't acknowledged in the paper. The authors should use their best judgment and recognize that individual actions in favor of transparency play an important role in developing norms that preserve the integrity of the community. Reviewers will be specifically instructed to not penalize honesty concerning limitations.
    \end{itemize}

\item {\bf Theory Assumptions and Proofs}
    \item[] Question: For each theoretical result, does the paper provide the full set of assumptions and a complete (and correct) proof?
    \item[] Answer: \answerYes{} 
    \item[] Justification: In Section \ref{sec: theoretical} and Appendix \ref{sec:b_exp},\ref{sec:c_exp},\ref{sec:d_exp}, \ref{sec:e_exp}, we provide full set of assumptions and proofs.
    \item[] Guidelines:
    \begin{itemize}
        \item The answer NA means that the paper does not include theoretical results. 
        \item All the theorems, formulas, and proofs in the paper should be numbered and cross-referenced.
        \item All assumptions should be clearly stated or referenced in the statement of any theorems.
        \item The proofs can either appear in the main paper or the supplemental material, but if they appear in the supplemental material, the authors are encouraged to provide a short proof sketch to provide intuition. 
        \item Inversely, any informal proof provided in the core of the paper should be complemented by formal proofs provided in the appendix or supplemental material.
        \item Theorems and Lemmas that the proof relies upon should be properly referenced. 
    \end{itemize}

    \item {\bf Experimental Result Reproducibility}
    \item[] Question: Does the paper fully disclose all the information needed to reproduce the main experimental results of the paper to the extent that it affects the main claims and/or conclusions of the paper (regardless of whether the code and data are provided or not)?
    \item[] Answer: \answerYes{} 
    \item[] Justification: In Sections \ref{sec:exp_case1} and Appendix \ref{sec:a_exp}, we provide a comprehensive description of the experimental setup, including detailed information on the datasets used, model architectures, and hyperparameter settings. This detailed disclosure ensures that other researchers can reliably reproduce the experimental results and validate the claims made in the paper.
    \item[] Guidelines:
    \begin{itemize}
        \item The answer NA means that the paper does not include experiments.
        \item If the paper includes experiments, a No answer to this question will not be perceived well by the reviewers: Making the paper reproducible is important, regardless of whether the code and data are provided or not.
        \item If the contribution is a dataset and/or model, the authors should describe the steps taken to make their results reproducible or verifiable. 
        \item Depending on the contribution, reproducibility can be accomplished in various ways. For example, if the contribution is a novel architecture, describing the architecture fully might suffice, or if the contribution is a specific model and empirical evaluation, it may be necessary to either make it possible for others to replicate the model with the same dataset, or provide access to the model. In general. releasing code and data is often one good way to accomplish this, but reproducibility can also be provided via detailed instructions for how to replicate the results, access to a hosted model (e.g.,, in the case of a large language model), releasing of a model checkpoint, or other means that are appropriate to the research performed.
        \item While NeurIPS does not require releasing code, the conference does require all submissions to provide some reasonable avenue for reproducibility, which may depend on the nature of the contribution. For example
        \begin{enumerate}
            \item If the contribution is primarily a new algorithm, the paper should make it clear how to reproduce that algorithm.
            \item If the contribution is primarily a new model architecture, the paper should describe the architecture clearly and fully.
            \item If the contribution is a new model (e.g.,, a large language model), then there should either be a way to access this model for reproducing the results or a way to reproduce the model (e.g.,, with an open-source dataset or instructions for how to construct the dataset).
            \item We recognize that reproducibility may be tricky in some cases, in which case authors are welcome to describe the particular way they provide for reproducibility. In the case of closed-source models, it may be that access to the model is limited in some way (e.g.,, to registered users), but it should be possible for other researchers to have some path to reproducing or verifying the results.
        \end{enumerate}
    \end{itemize}

\item {\bf Open access to data and code}
    \item[] Question: Does the paper provide open access to the data and code, with sufficient instructions to faithfully reproduce the main experimental results, as described in supplemental material?
    \item[] Answer: \answerYes{} 
    \item[] Justification: We upload the code with environment in the supplemental materials. The datasets used are all open-source datasets.
    \item[] Guidelines:
    \begin{itemize}
        \item The answer NA means that paper does not include experiments requiring code.
        \item Please see the NeurIPS code and data submission guidelines (\url{https://nips.cc/public/guides/CodeSubmissionPolicy}) for more details.
        \item While we encourage the release of code and data, we understand that this might not be possible, so “No” is an acceptable answer. Papers cannot be rejected simply for not including code, unless this is central to the contribution (e.g.,, for a new open-source benchmark).
        \item The instructions should contain the exact command and environment needed to run to reproduce the results. See the NeurIPS code and data submission guidelines (\url{https://nips.cc/public/guides/CodeSubmissionPolicy}) for more details.
        \item The authors should provide instructions on data access and preparation, including how to access the raw data, preprocessed data, intermediate data, and generated data, etc.
        \item The authors should provide scripts to reproduce all experimental results for the new proposed method and baselines. If only a subset of experiments are reproducible, they should state which ones are omitted from the script and why.
        \item At submission time, to preserve anonymity, the authors should release anonymized versions (if applicable).
        \item Providing as much information as possible in supplemental material (appended to the paper) is recommended, but including URLs to data and code is permitted.
    \end{itemize}

\item {\bf Experimental Setting/Details}
    \item[] Question: Does the paper specify all the training and test details (e.g.,, data splits, hyperparameters, how they were chosen, type of optimizer, etc.) necessary to understand the results?
    \item[] Answer: \answerYes{} 
    \item[] Justification: In Sections \ref{sec:exp_case1} and Appendix \ref{sec:a_exp}, we provide a comprehensive description of the experimental setup, including detailed information on the datasets used, model architectures, and hyperparameter settings. More details can be found in the code in the supplemental materials.

\item {\bf Experiment Statistical Significance}
    \item[] Question: Does the paper report error bars suitably and correctly defined or other appropriate information about the statistical significance of the experiments?
    \item[] Answer: \answerNo{} 
    \item[] Justification: All the results are reported in terms of learning curves, and each figure includes many plots, so error bars are not reported. But we do run the experiments multiple times and observe very similar performance in terms of NRC, testing loss, etc.
    \item[] Guidelines:
    \begin{itemize}
        \item The answer NA means that the paper does not include experiments.
        \item The authors should answer "Yes" if the results are accompanied by error bars, confidence intervals, or statistical significance tests, at least for the experiments that support the main claims of the paper.
        \item The factors of variability that the error bars are capturing should be clearly stated (for example, train/test split, initialization, random drawing of some parameter, or overall run with given experimental conditions).
        \item The method for calculating the error bars should be explained (closed form formula, call to a library function, bootstrap, etc.)
        \item The assumptions made should be given (e.g.,, Normally distributed errors).
        \item It should be clear whether the error bar is the standard deviation or the standard error of the mean.
        \item It is OK to report 1-sigma error bars, but one should state it. The authors should preferably report a 2-sigma error bar than state that they have a 96\% CI, if the hypothesis of Normality of errors is not verified.
        \item For asymmetric distributions, the authors should be careful not to show in tables or figures symmetric error bars that would yield results that are out of range (e.g., negative error rates).
        \item If error bars are reported in tables or plots, The authors should explain in the text how they were calculated and reference the corresponding figures or tables in the text.
    \end{itemize}

\item {\bf Experiments Compute Resources}
    \item[] Question: For each experiment, does the paper provide sufficient information on the computer resources (type of compute workers, memory, time of execution) needed to reproduce the experiments?
    \item[] Answer: \answerYes{} 
    \item[] Justification: In Appendix \ref{sec:a_exp}, we provide the details for computation resources.
    \item[] Guidelines:
    \begin{itemize}
        \item The answer NA means that the paper does not include experiments.
        \item The paper should indicate the type of compute workers CPU or GPU, internal cluster, or cloud provider, including relevant memory and storage.
        \item The paper should provide the amount of compute required for each of the individual experimental runs as well as estimate the total compute. 
        \item The paper should disclose whether the full research project required more compute than the experiments reported in the paper (e.g.,, preliminary or failed experiments that didn't make it into the paper). 
    \end{itemize}
    
\item {\bf Code Of Ethics}
    \item[] Question: Does the research conducted in the paper conform, in every respect, with the NeurIPS Code of Ethics \url{https://neurips.cc/public/EthicsGuidelines}?
    \item[] Answer: \answerYes{} 
    \item[] Justification: We are sure that the research presented in this paper adheres to the NeurIPS Code of Ethics in all respects.
    \item[] Guidelines:
    \begin{itemize}
        \item The answer NA means that the authors have not reviewed the NeurIPS Code of Ethics.
        \item If the authors answer No, they should explain the special circumstances that require a deviation from the Code of Ethics.
        \item The authors should make sure to preserve anonymity (e.g.,, if there is a special consideration due to laws or regulations in their jurisdiction).
    \end{itemize}

\item {\bf Broader Impacts}
    \item[] Question: Does the paper discuss both potential positive societal impacts and negative societal impacts of the work performed?
    \item[] Answer: \answerNA{} 
    \item[] Justification: This paper mainly focuses on showing and understanding the neural regression collapse phenomena observed in practical neural networks and unconstrained feature models. No potential negative societal impact is expected of this work. 
    \item[] Guidelines:
    \begin{itemize}
        \item The answer NA means that there is no societal impact of the work performed.
        \item If the authors answer NA or No, they should explain why their work has no societal impact or why the paper does not address societal impact.
        \item Examples of negative societal impacts include potential malicious or unintended uses (e.g.,, disinformation, generating fake profiles, surveillance), fairness considerations (e.g.,, deployment of technologies that could make decisions that unfairly impact specific groups), privacy considerations, and security considerations.
        \item The conference expects that many papers will be foundational research and not tied to particular applications, let alone deployments. However, if there is a direct path to any negative applications, the authors should point it out. For example, it is legitimate to point out that an improvement in the quality of generative models could be used to generate deepfakes for disinformation. On the other hand, it is not needed to point out that a generic algorithm for optimizing neural networks could enable people to train models that generate Deepfakes faster.
        \item The authors should consider possible harms that could arise when the technology is being used as intended and functioning correctly, harms that could arise when the technology is being used as intended but gives incorrect results, and harms following from (intentional or unintentional) misuse of the technology.
        \item If there are negative societal impacts, the authors could also discuss possible mitigation strategies (e.g.,, gated release of models, providing defenses in addition to attacks, mechanisms for monitoring misuse, mechanisms to monitor how a system learns from feedback over time, improving the efficiency and accessibility of ML).
    \end{itemize}
    
\item {\bf Safeguards}
    \item[] Question: Does the paper describe safeguards that have been put in place for responsible release of data or models that have a high risk for misuse (e.g.,, pretrained language models, image generators, or scraped datasets)?
    \item[] Answer: \answerNA{} 

\item {\bf Licenses for existing assets}
    \item[] Question: Are the creators or original owners of assets (e.g.,, code, data, models), used in the paper, properly credited and are the license and terms of use explicitly mentioned and properly respected?
    \item[] Answer: \answerYes{} 
    \item[] Justification: All the owner of assets are properly cited in the reference and main body of our paper.
    \item[] Guidelines:
    \begin{itemize}
        \item The answer NA means that the paper does not use existing assets.
        \item The authors should cite the original paper that produced the code package or dataset.
        \item The authors should state which version of the asset is used and, if possible, include a URL.
        \item The name of the license (e.g.,, CC-BY 4.0) should be included for each asset.
        \item For scraped data from a particular source (e.g.,, website), the copyright and terms of service of that source should be provided.
        \item If assets are released, the license, copyright information, and terms of use in the package should be provided. For popular datasets, \url{paperswithcode.com/datasets} has curated licenses for some datasets. Their licensing guide can help determine the license of a dataset.
        \item For existing datasets that are re-packaged, both the original license and the license of the derived asset (if it has changed) should be provided.
        \item If this information is not available online, the authors are encouraged to reach out to the asset's creators.
    \end{itemize}

\item {\bf New Assets}
    \item[] Question: Are new assets introduced in the paper well documented and is the documentation provided alongside the assets?
    \item[] Answer: \answerNA{} 


\item {\bf Crowdsourcing and Research with Human Subjects}
    \item[] Question: For crowdsourcing experiments and research with human subjects, does the paper include the full text of instructions given to participants and screenshots, if applicable, as well as details about compensation (if any)? 
    \item[] Answer: \answerNA{} 

\item {\bf Institutional Review Board (IRB) Approvals or Equivalent for Research with Human Subjects}
    \item[] Question: Does the paper describe potential risks incurred by study participants, whether such risks were disclosed to the subjects, and whether Institutional Review Board (IRB) approvals (or an equivalent approval/review based on the requirements of your country or institution) were obtained?
    \item[] Answer: \answerNA{} 

\end{enumerate}

\end{document}